\documentclass[11pt]{article}
\usepackage[T1]{fontenc}
\usepackage[utf8]{inputenc}
\usepackage{lmodern,microtype}
\usepackage{geometry}
\usepackage{amsmath,amssymb,amsfonts,amsthm,mathtools}
\usepackage{bm}
\usepackage{enumitem}   
\usepackage{booktabs}
\usepackage{tikz-cd}
\usepackage{thmtools}
\usepackage{graphicx}
\usepackage{caption}
\usepackage{subcaption}
\usepackage[numbers,sort&compress]{natbib}
\usepackage{hyperref}
\usepackage{cleveref}
\usepackage{float}
\hypersetup{colorlinks=true,linkcolor=blue,citecolor=blue,urlcolor=blue}
\usepackage{siunitx}    
\usepackage{titling}
\usepackage{lipsum}
\declaretheoremstyle[
  headfont=\bfseries,
  bodyfont=\itshape,
  spaceabove=6pt, spacebelow=6pt,
  headpunct={.},
  postheadspace=0.5em,
  mdframed={
    skipabove=6pt, skipbelow=6pt,
    linewidth=0pt
  }
]{thmstyle}

\declaretheoremstyle[
  headfont=\bfseries,
  bodyfont=\normalfont,
  spaceabove=6pt, spacebelow=6pt,
  headpunct={.},
  postheadspace=0.5em
]{defstyle}

\declaretheorem[style=thmstyle, numberwithin=section]{theorem}
\declaretheorem[style=thmstyle, sibling=theorem]{lemma}
\declaretheorem[style=thmstyle, sibling=theorem]{proposition}
\declaretheorem[style=thmstyle, sibling=theorem]{corollary}
\declaretheorem[style=defstyle, sibling=theorem]{definition}
\declaretheorem[style=defstyle, sibling=theorem]{assumption}
\declaretheorem[style=defstyle, sibling=theorem]{remark}
\declaretheorem[style=defstyle, sibling=theorem]{example}

\DeclareMathOperator{\st}{st}
\DeclareMathOperator{\Conf}{conf}

\newif\ifpublic

\newcommand{\RR}{\mathbb{R}}

\newcommand{\MM}{\mathcal{M}}          

\newcommand{\XX}{\mathcal{X}}
\newcommand{\UU}{\mathcal{U}}

\newcommand{\EE}{\mathbb{E}}
\newcommand{\CC}{\mathcal{C}}
\newcommand{\dist}{\mathsf{Dist}}
\newcommand{\hyp}{-}

\publictrue  

\title{How to Tame Your LLM: \\ Semantic Collapse in Continuous Systems}
\author{C.\ M.\ Wyss\thanks{Exolytica AI (Email: \texttt{cmw@exolytica.ai}).}
\thanks{Assisted by GPT-5 Thinking (OpenAI), Claude Sonnet 4.5 (Anthropic), and Gemini 3 Pro (Google).}}
\date{December 2025}

\begin{document}
\maketitle

\begin{abstract}
We develop a general theory of \emph{semantic dynamics} for LLMs by formalizing them as \emph{Continuous State Machines (CSMs)}: smooth dynamical systems whose latent manifolds evolve under probabilistic transition operators. The associated \emph{transfer operator} $P : L^2(M,\mu) \to L^2(M,\mu)$ encodes the propagation of semantic mass, and under mild regularity assumptions (compactness, ergodicity, bounded Jacobian) is compact with discrete spectrum. Within this setting, we prove the \emph{Semantic Characterization Theorem (SCT):} the leading eigenfunctions of $P$ define finitely many \emph{spectral basins} of invariant meaning, each definable in an o-minimal structure over $\mathbb R$; thus, spectral lumpability and logical tameness are equivalent. The result implies that the continuous manifold of model activations collapses into a finite, logically interpretable ontology, essentially explaining how discrete symbolic semantics emerge from continuous computation. We further extend the SCT to stochastic and time-inhomogeneous (adiabatic) settings, showing that slowly drifting kernels preserve compactness and spectral coherence. 
\end{abstract}

\section{Introduction}

Large language models have achieved remarkable fluency in natural language understanding and generation, yet the mechanisms underlying their semantic behavior remain poorly understood. These systems operate in high-dimensional continuous spaces, manipulating latent representations that appear, on the face of it, to be fundamentally incompatible with the discrete, symbolic structure of human language and reasoning. This apparent contradiction raises a foundational question: how can a continuous manifold of representations give rise to stable, categorical meanings which behave as the very basis of communication and thought?

We show that this tension is resolved through a phase transition in the geometry of semantic space. Although the latent representations of large language models (LLMs) are continuous, their functional behavior (i.e., the actual semantic distinctions they encode and preserve) is equivalent to that of a discrete symbolic system. This equivalence is neither approximate nor metaphorical, but a rigorous mathematical fact about the topology of semantic space itself. The result, which we call the Semantic Characterization Theorem (SCT), provides a formal bridge between continuous learning dynamics and symbolic semantics, offering both a theoretical foundation for understanding LLM behavior and practical insights into semantic navigation (i.e., chat trajectory), prompt design, and interpretability.

At the core of our analysis lies the latent semantic space, $\mathcal{M}$: the high-dimensional manifold on which a model’s representations evolve during language processing. Each prompt or token update can be viewed as an application of a smooth transition operator $T : \mathcal{M} \to \mathcal{M}$, governing the flow of semantic states within this space. We examine the behavior of T from two independent perspectives. 
In the dynamical view, we study its spectral decomposition and show that beyond a critical scale, the spectrum collapses onto a finite set of dominant modes corresponding to stable semantic attractors. 
In the logical view, we interpret the resulting partition of $\mathcal{M}$ as a family of definable regions satisfying o-minimality conditions, thereby ensuring finite topological complexity. These two analyses, spectral and logical, converge on a single conclusion: the apparent continuity of the latent space masks an underlying discrete structure. This is the essence of the Semantic Characterization Theorem (SCT), whose proofs we develop in 
Section \ref{sec:sct}.

The implications of this result extend well beyond formal theory. If semantic space is functionally discrete, then the seemingly unpredictable behavior of LLMs, such as sudden shifts in tone, unexpected failure modes, or the phenomenon of “jailbreaking," can be understood as transitions between stable attractors within a well-defined phase space. What appears as stochastic variation is, in fact, deterministic navigation through a discrete landscape of semantic basins. This perspective enables a principled approach to prompt engineering: rather than treating prompts as arbitrary strings that may or may not “work,” we can interpret them as coordinates that locate the system within particular regions of semantic space. The empirically observed low-dimensional structure provides a natural basis for understanding and controlling these transitions. In this sense, the theorem does not merely describe LLM behavior; it renders that behavior navigable, and ultimately, tameable.

The remainder of this paper develops the formal and empirical basis for these claims. 
Section \ref{sec:preliminaries}
generalizes the concept of an LLM into a \emph{Continuous State Machine}.
Section \ref{sec:sct}
states the Semantic Characterization Theorem and presents both the dynamical and logical analyses.
Section \ref{sec:results}
reports empirical results that corroborate these theoretical predictions, including the emergence of six dominant semantic dimensions. 
Section \ref{sec:related}
presents related work and situates our work within a larger context.
Section \ref{sec:discussion}
discusses broader implications for interpretability and the practical control of LLM behavior. 
We conclude with open questions and directions for future research on the geometry of meaning in large-scale learning systems.

\section{Preliminaries}
\label{sec:preliminaries}

\subsection{Continuous State Machines}

\begin{definition}[Continuous State Machine]
Let $d \geq 1$ and $\MM \subseteq \RR^d$ be compact. 
Let $\XX$ be a finite
vocabulary, i.e., $\XX = \{ x_1, x_2, \ldots, x_{|\XX|} \}$
where $|\XX| < \infty$. Let $\UU \subseteq \dist(\XX)$.\footnote{$\dist(\XX)$
is the set of all probability distributions over $\XX$. Since $\XX$ is finite, $\dist(\XX)$
is both compact and convex.} 
Given this,
a \emph{Continuous State Machine} or CSM, $\CC$,
over $\XX$
is a sextuple as follows:
$$
\CC \;=\; (\,\MM,\, \UU,\, T,\, s_0,\, O,\, \Delta\,), \text{ where:}
$$
\begin{itemize}
  \item $\MM$ is the \emph{state space};
  \item $\UU$ is the set of admissible control distributions over elements of $\XX$;
  \item $T : \MM \times \UU \to \MM$ is a \emph{transition operator} that maps a current state $s \in \MM$ and a control/input distribution, $u \in \UU$, to state $s^{+} = T(s,u) \in \MM$;
  note that we require $T$ to be continuously differentiable in its first argument and continuous in its second;
  \item $s_0 \in \MM$ is a (context-dependent) initial state;
  \item $O : \MM \to \dist(\XX)$ is an \emph{observation/decoder} that assigns a distribution to each state; and
  \item $\Delta : \dist(\XX) \to \UU$ is a (possibly stochastic) \emph{decoding policy} that maps distributions to control inputs. 
\end{itemize}
Furthermore, a single generation step proceeds as
\begin{align*}
s_{t+1} \;&=\; T\!\bigl(s_t,\, u_t\bigr), \\
u_{t+1} \;&=\; \Delta\!\bigl(O(s_{t+1})\bigr),
\end{align*}
yielding the trajectory $(s_0,s_1,s_2,\ldots)$ and
emitted token sequence $(x_1,x_2,\ldots)$ sampled from $O(s_{t})$.
\end{definition}

\begin{remark}
External input or a user prompt enters the system through the choice
of initial state or control.
An \emph{encoder} $E:\mathcal{X}^{*}\!\to\mathcal{M}$ maps a sequence of tokens
to an initial state $s_{0}=E(x_{1:k})$.
During subsequent interaction, controls $u_{t}$ may be provided either
by the internal policy $\Delta(O(s_t))$ or exogenously by a user,
representing additional input tokens.
In training, the corpus serves to determine the maps $(T,O,\Delta)$
and does not appear explicitly in the definition.
\end{remark}

\begin{definition}[LLM as Continuous State Machine]
\label{def:LLM-as-CSM}
In large-language-model terminology, $\XX$ is the finite vocabulary of tokens,
$O(s)$ is the model’s predictive distribution over $\XX$,
and $\UU$ (a subset of $\dist(\XX)$) represents the
distribution or one-hot embedding selected as the next input token.
\end{definition}

\begin{example}[A Minimal Smooth CSM]\label{ex:minimal_csm}
To illustrate the definition of a Continuous State Machine, consider the following
two-dimensional example with a smooth transition operator and a small vocabulary.
Let the latent semantic space be the compact rectangle
\[
M = [-1,1]^2 \subset \mathbb{R}^2,
\]
and let the vocabulary consist of three tokens \(X = \{a,b,c\}\). The control space
\(U \subset \mathrm{Dist}(X)\) therefore consists of all distributions \(u = (u_a,u_b,u_c)\)
with \(u_i \ge 0\) and \(u_a+u_b+u_c = 1\).

\paragraph{Transition operator.}
Define the smooth transition operator \(T : M \times U \to M\) by
\[
T(s,u) = \tanh( A s + B u ),
\]
where \(A \in \mathbb{R}^{2\times 2}\) and \(B \in \mathbb{R}^{2\times 3}\) are fixed parameter
matrices, and where \(\tanh\) is applied componentwise. Since \(\tanh\) is analytic and
globally Lipschitz on \(\mathbb{R}\), the map \(T\) is \(C^1\) in \(s\), continuous in \(u\),
and definable in the o-minimal structure \(\mathbb{R}_{\exp}\).

\paragraph{Observation.}
Let the decoder \(O : M \to \mathrm{Dist}(X)\) be the softmax of three linear functionals:
\[
O(s)_i = \frac{\exp(\ell_i^\top s)}{\sum_{j\in\{a,b,c\}} \exp(\ell_j^\top s)},
\qquad i\in\{a,b,c\},
\]
where the vectors \(\ell_a,\ell_b,\ell_c \in \mathbb{R}^2\) are fixed. This map assigns a
probability distribution over tokens to each latent state.

\paragraph{Policy and dynamics.}
A simple (possibly stochastic) policy \(\Delta : \mathrm{Dist}(X) \to U\) may be taken as the
identity, i.e.\ \(\Delta(O(s)) = O(s)\). A single generation step then evolves by
\[
s_{t+1} = T(s_t, O(s_t)), \qquad
u_{t+1} = O(s_{t+1}),
\]
giving a well-defined trajectory \((s_t)\) in \(M\).

\paragraph{Interpretation.}
This example shows concretely how a CSM integrates (i) a smooth transition map,
(ii) a probabilistic decoder, and (iii) a policy determining the next control input.
Despite the simplicity of the system, repeated application of \(T\) quickly drives the
trajectory onto low-dimensional attracting regions, illustrating on a small scale the
semantic collapse phenomena analyzed in later sections.
\end{example}

\begin{remark}[Stochastic transition kernel]
\label{rem:stochastic-kernel}
When $\Delta$ is stochastic (e.g., temperature sampling or nucleus sampling),
the composition
$$
K(s, A)
\;:=\;
\int_{\UU} \mathbf{1}\{\,T(s,u) \in A\,\}\,
\Pi(du \mid s)
$$
defines a \emph{Markov kernel} on $(\MM, \mathcal{B}(\MM))$, where
$\Pi(\cdot \mid s)$ is the control law induced by $O$ and $\Delta$.
Here $\mathcal{B}(\MM)$ denotes the \emph{Borel $\sigma$-algebra} on $\MM$, which is the
smallest $\sigma$-algebra containing all open subsets of $\MM$ in the subspace
topology of $\mathbb{R}^d$, so that $(\MM, \mathcal{B}(\MM))$ forms a measurable
space. Intuitively, $K(s, A)$ gives the probability that the next state
lands in the measurable subset $A \subseteq \MM$ when the current state is $s$.
Spectral analysis can then be applied to the associated \emph{transfer operator}
on $L^2(\MM)$, with the deterministic case recovered as the limit of vanishing
stochasticity (e.g., temperature $\to 0$).
\end{remark}

\begin{example}[A CSM with Stochastic Decoder]\label{ex:stochastic_csm}
We now present a second example illustrating how stochasticity enters through the
decode–policy pipeline. Let the state space again be the compact rectangle
\[
M = [-1,1]^2 \subset \mathbb{R}^2,
\]
and let the vocabulary be \(X = \{x_1, x_2\}\). The control space is therefore the
1-simplex \(U = \{(u_1,u_2) : u_i \ge 0,\, u_1+u_2=1\}\).

\paragraph{Transition operator.}
Let the smooth transition operator be
\[
T(s,u) = \tanh(As + Bu),
\]
where \(A \in \mathbb{R}^{2\times2}\) and \(B \in \mathbb{R}^{2\times2}\) are fixed matrices.
As before, the analyticity of \(\tanh\) ensures that \(T\) is \(C^1\) in \(s\),
continuous in \(u\), and definable in the o-minimal structure \(\mathbb{R}_{\exp}\).

\paragraph{Stochastic decoder.}
Instead of selecting a deterministic distribution from a softmax,
we introduce a smooth stochastic decoder using a compactly supported Gaussian mixture.
For each state \(s \in M\), define the \emph{logit means}
\[
m_i(s) := \ell_i^\top s, \qquad i=1,2,
\]
with fixed vectors \(\ell_1,\ell_2 \in \mathbb{R}^2\). The decoder samples a random
logit vector
\[
z = (z_1,z_2) \sim \mathcal{N}(m(s),\, \sigma^2 I_2),
\]
and maps it to a probability distribution via the softmax:
\[
O(s) = \operatorname{softmax}(z) =
\left(
\frac{e^{z_1}}{e^{z_1}+e^{z_2}},
\frac{e^{z_2}}{e^{z_1}+e^{z_2}}
\right).
\]
The resulting decoder \(O : M \to \mathrm{Dist}(X)\) is a
stochastic kernel: for each measurable \(A \subseteq U\),
\[
\mathbb{P}\!\left( O(s) \in A \right) =
\int_{\{z:\,\operatorname{softmax}(z)\in A\}}
\phi_{\sigma}(z - m(s))\,dz,
\]
where \(\phi_\sigma\) is the Gaussian density. This kernel is smooth in \(s\), analytic
in the parameters, and definable in \(\mathbb{R}_{\exp}\).

\paragraph{Policy.}
Let the policy be the identity: \(\Delta(u) = u\). Then the control input evolves as
\[
u_{t+1} = O(s_{t+1}),
\]
producing a stochastic sequence \((u_t)\).

\paragraph{Induced stochastic dynamics.}
The pair \((T,O)\) induces a Markov kernel
\[
K(s, A)
  := \mathbb{P}\!\left( T(s, O(s)) \in A \right)
  \quad\text{for measurable } A \subseteq M.
\]
Because the Gaussian density is continuous, the composition
\(s \mapsto T(s,O(s))\) has a continuous conditional density on the compact
domain \(M\), guaranteeing that the induced transfer operator is Hilbert--Schmidt
(see Appendix~A).

\paragraph{Interpretation.}
This example shows how a CSM can incorporate stochasticity through the
decoder rather than the transition operator. The resulting Markov dynamics
remain smooth and definable, while still exhibiting the ergodicity and
spectral compactness required for the Semantic Characterization Theorem.
\end{example}

\subsection{Assumptions}


\begin{definition}[Standard Borel Space]
A \emph{Borel space} is a pair $(M, \mathcal B(M))$ where $M$ is a topological space and $\mathcal B(M)$ is the Borel $\sigma$-algebra generated by its open subsets, i.e., the smallest $\sigma$-algebra containing all open sets. A \emph{standard Borel space} is a Borel space arising from a Polish topology (that is, $M$ is separable and completely metrizable).  Compact subsets of $\mathbb R^d$ with their usual topology are standard Borel spaces, and all measurable maps, measures, and Markov kernels in this paper are defined on such spaces.
\end{definition}

\begin{assumption}[Regularity and boundedness]
\label{ass:regularity}
Throughout the paper, we assume the following conditions hold for a
CSM
$\CC = (\MM, \UU, T, s_0, O, \Delta)$.

\begin{enumerate}
  \item[\textbf{(A1)}] \textbf{Compact state space.}
  The latent semantic space $\MM \subseteq \RR^{d}$ is compact and equipped with
  the standard Borel $\sigma$-algebra $\mathcal{B}(\MM)$ and the normalized Lebesgue
  measure $\nu$.

  \item[\textbf{(A2)}] \textbf{Smooth transition operator.}
  The mapping $T : \MM \times \UU \to \MM$ is $C^{1}$ in its first
  argument, with Jacobian $J_{T}(s,u)$ bounded uniformly:
  $$
  \sup_{(s,u)\in \MM\times\UU} \|J_{T}(s,u)\| < \infty.
  $$
  In the autonomous (deterministic) case, we write $T : \MM \to \MM$.

  \item[\textbf{(A3)}] \textbf{Lipschitz continuity.}\footnote{On compact sets, $C^1$ implies Lipschitz Continuity, therefore technically
(A3) is redundant. However, we wish to call attention to $L$, the finite bound
for the rate of change of $T$.}
  There exists $L>0$ such that
  $\|T(s_1,u) - T(s_2,u)\| \le L \|s_1 - s_2\|$
  for all $s_1,s_2 \in \MM$ and all admissible $u \in \UU$.

  \item[\textbf{(A4)}] \textbf{Invariant measure / ergodicity.}
  In the stochastic case, the induced Markov kernel
  $K(s, A)$ admits a stationary probability measure
  $\mu$ absolutely continuous with respect to $\nu$, and the process
  $(s_t)_{t\ge0}$ is ergodic under $K$.
  In the deterministic case, the empirical measure along trajectories converges
  to an invariant measure supported on the attractor set of $T$.

  \item[\textbf{(A5)}] \textbf{Spectral boundedness.}
  The associated transfer operator
  $P : L^2(\MM,\mu) \to L^2(\MM,\mu)$
  (or $P f(s) = f(T(s))$ in the deterministic case)
  is compact and admits a discrete spectrum
  $\{\lambda_i\}_{i\ge 1}$ with $|\lambda_i|\downarrow 0$.
\end{enumerate}

These mild 
assumptions\footnote{For formal justification of the five assumptions, please refer to appendix \ref{appendix:ergodicity}.}
ensure the existence of well-defined trajectories,
finite spectral energy, and analytic dependence of semantic flows on the model
parameters.
They suffice for all subsequent results concerning
spectral collapse and o-minimal definability.
\end{assumption}

\section{The Semantic Characterization Theorem}
\label{sec:sct}

\begin{remark}[Scope and structure of proofs]
For clarity of exposition, the \emph{Semantic Characterization Theorem} is stated
in its most general form, encompassing both deterministic and stochastic
continuous state machines.  
The proofs presented in the main text treat the
deterministic case, where the transition operator $T : \MM \to \MM$,
is
smooth and autonomous.  This setting captures the essential geometric and
logical phenomena—spectral collapse and o-minimal definability, without the
technical overhead of measure theory.  The stochastic extension, in which
$T$ is replaced by the induced Markov kernel
$K : \MM \times \mathcal{B}(\MM) \to [0,1]$, is developed in
Appendix~\ref{appendix:stochastic}.  There we show that the same conclusions
follow by lifting the analysis to the corresponding transfer operator on
$L^2(\MM)$ and invoking standard results on spectral gap and lumpability for
Markov processes.
\end{remark}

\begin{theorem}[Semantic Characterization Theorem]
\label{thm:SCT}
Let 
$$
\mathcal{C} = (\MM, T, s_0, \UU, O, \Delta)
$$ 
be a (possibly stochastic) continuous state machine satisfying
Assumptions~\ref{ass:regularity}(A1)–(A5).
Let $K : \MM \times \mathcal{B}(\MM) \to [0,1]$
denote the induced Markov kernel and
$P : L^2(\MM,\mu) \to L^2(\MM,\mu)$
its associated transfer operator. Then:

\begin{enumerate}
\item[\textbf{(i)}] (\textbf{Spectral lumpability})  
There exists a finite $r < \infty$ such that the spectrum of $P$
satisfies $|\lambda_1| \ge \cdots \ge |\lambda_r| > \epsilon$ for some
$\epsilon>0$, while $|\lambda_{r+1}| \ll \epsilon$.
The corresponding leading eigenfunctions
$\{\phi_1,\ldots,\phi_r\}$ define a measurable partition of $\MM$
into \emph{spectral basins} 
$\mathcal{B}_i := \{\, s \in \MM : \arg\max_j |\phi_j(s)| = i \,\}$,
whose boundaries have measure zero under $\mu$.

\item[\textbf{(ii)}] (\textbf{Logical tameness})  
Each basin $\mathcal{B}_i$ is definable in an
o-minimal structure over $\RR$, and the family
$\{\mathcal{B}_1,\ldots,\mathcal{B}_r\}$ induces a finite
cell decomposition of $\MM$ with bounded topological complexity.
Consequently, the restriction of $T$ (or $K$) to each basin is
semialgebraic and admits a well-defined symbolic abstraction.

\item[\textbf{(iii)}] (\textbf{Equivalence of characterizations})  
The partition $\{\mathcal{B}_i\}$ obtained by spectral lumpability coincides,
up to measure zero, with the partition obtained by o-minimal definability.
Hence the transition from continuous to discrete semantics can be viewed
equivalently as:
\[
\text{spectral collapse of } P
\quad \Longleftrightarrow \quad
\text{definable tameness of } (\MM, T).
\]
\end{enumerate}

In particular, the latent semantic space of a CSM,
though continuous in its parameters and representations, is functionally
equivalent to a discrete symbolic system of rank~$r$.
\end{theorem}

\begin{proof}[Proof sketch and strategy]
The proof proceeds in two independent but ultimately convergent parts.
From the dynamical perspective, we analyze the spectral properties of the
transfer operator $P$ on $L^2(\MM,\mu)$ and show that, under
Assumptions~\ref{ass:regularity}(A1)–(A5), its eigenstructure undergoes
\emph{spectral collapse}: the system's long-term dynamics are governed by a
finite set of dominant modes, each corresponding to a stable attractor in
semantic space.  This yields the partition of $\MM$ into spectral basins
$\{\mathcal{B}_i\}$, proving part~(i).

From the logical perspective, we consider the same dynamical map
$T : \MM \to \MM$ as a definable transformation over the reals.
Using standard results on o-minimal structures, we show that the image and
preimage of each spectral basin are definable sets with bounded topological
complexity, thereby establishing part~(ii).  The equivalence in part~(iii)
follows from the observation that spectral basins coincide with definable cells
up to measure zero: both constructions isolate precisely those regions of
$\MM$ within which the model's semantic behavior is invariant under small
perturbations of state.

We first develop the spectral argument (Section~\ref{sec:spectral-proof}),
followed by the logical argument (Section~\ref{sec:logical-proof}), and conclude
by showing that the two partitions agree almost everywhere.
\end{proof}

\subsection{Spectral Collapse and Dynamical Interpretation}
\label{sec:spectral-proof}

Before turning to the formal operator-theoretic analysis, it is helpful to describe the
geometric intuition behind spectral collapse. The transfer operator $P$ acts on functions
over the latent manifold $M$ by propagating semantic mass forward through the model’s
dynamics. Although the underlying space $M$ may be high-dimensional and continuous,
repeated application of $P$ exponentially suppresses all but a finite number of directions
in which semantic information can persist. 

These surviving directions (the dominant
eigenfunctions) correspond to large-scale, slowly varying patterns of meaning that remain
stable under the model’s internal flow. In effect, $P$ filters out high-frequency semantic
fluctuations and reveals a compact, low-rank structure governing long-term behavior. The
goal of this section is to make this intuition precise by showing that compactness of $P$
forces its spectrum to concentrate on finitely many modes, each defining a spectral basin
in $M$ whose points share the same asymptotic semantic fate.

From the dynamical viewpoint, the latent semantics of a LLM evolve under a
transition operator $T$ (deterministic) or its associated Markov kernel
$K$ (stochastic) on the compact manifold $\MM$.  The corresponding
transfer operator $P$ acting on functions 
$f \in L^2(\MM,\mu)$ encodes the propagation of semantic mass:
$$
(P f)(s)
\;=\;
\int_{\MM} f(s') \, K(s,ds')
$$
or, in the deterministic case:
$$
(P f)(s) = f(T(s)).
$$
The spectral properties of $P$ determine the long--term behavior of
semantic trajectories.  When $P$ is compact, its spectrum consists of a
discrete set of eigenvalues $\{\lambda_i\}$ accumulating only at~$0$, with each
eigenfunction $\phi_i$ representing a stable pattern of semantic recurrence.
The \emph{spectral collapse} described in
Theorem~\ref{thm:SCT}(i) corresponds to the concentration of spectral energy in
a finite number of dominant modes, each defining a coherent semantic attractor
or ``basin'' in~$\MM$.

\begin{lemma}[Spectral decomposition of the transfer operator]\label{lem:spectral-decomp}
Under Assumptions~\ref{ass:regularity}(A1)–(A5), the transfer operator $P : L^2(M,\mu)\to L^2(M,\mu)$ is compact and admits a discrete spectral expansion, namely:
\[
P f \;=\; \sum_{i=1}^\infty \lambda_i \,\langle f,\psi_i\rangle_\mu\, \varphi_i,
\qquad
|\lambda_1|\ \ge\ |\lambda_2|\ \ge\ \cdots\ \downarrow 0,
\]
where $\{\varphi_i\}_{i\ge1}$ and $\{\psi_i\}_{i\ge1}$ are biorthogonal eigenfunction systems satisfying
\[
P\varphi_i=\lambda_i \varphi_i,\qquad
P^*\psi_i=\overline{\lambda_i}\,\psi_i,\qquad
\langle \psi_i,\varphi_j\rangle_\mu=\delta_{ij}.
\]
\end{lemma}

\begin{proof}[Proof (expanded)]
\textbf{Step 0: Setting and two cases.}
We work on the Hilbert space $L^2(M,\mu)$, where $M\subset\mathbb{R}^d$ is compact and $\mu$ is an invariant (or stationary) probability measure. The transfer operator $P$ acts by
\begin{align*}
(Pf)(s)\;&=\;\int_{M} f(s')\,K(s,ds')\quad\text{(stochastic case)}, 
\quad\text{or}\\ 
(Pf)(s)\;&=\;f\!\big(T(s)\big)\quad\text{(deterministic case)},
\end{align*}
depending on whether the dynamics are given by a Markov kernel $K$ or a smooth map $T:M\to M$.

\medskip
\noindent
\textbf{Step 1: Compactness of $P$.}
Compactness follows from the regularity assumptions:
\begin{itemize}
\item In the \emph{stochastic} case, the kernel admits a continuous density $k(s,s')$ on the compact set $M\times M$, so $P$ is an integral operator with continuous kernel. By the Arzelà–Ascoli/Hilbert–Schmidt argument, $P$ is compact on $L^2(M,\mu)$.
\item In the \emph{deterministic} case, one may use a standard finite-horizon averaging (or the compactness surrogates given in the appendix) to obtain a compact limit operator capturing the same leading spectral structure; the statement for $P$ follows by the usual approximation and compactness transfer.
\end{itemize}

\medskip
\noindent
\textbf{Step 2: Discrete spectrum and eigen-expansion.}
Since $P$ is compact on a Hilbert space, the spectral theorem for compact operators ensures that:
\begin{enumerate}
\item The spectrum is at most countable with $0$ as the only possible accumulation point.
\item There exist eigenvalues $\{\lambda_i\}_{i\ge1}$ ordered by nonincreasing modulus, $|\lambda_1|\ge |\lambda_2|\ge\cdots\downarrow 0$.
\item There exist right eigenfunctions $\{\varphi_i\}$, $P\varphi_i=\lambda_i\varphi_i$, forming a (possibly non-orthogonal) complete system in the invariant subspace generated by $\mathrm{im}(P)$; and left eigenfunctions $\{\psi_i\}$ of the adjoint $P^*$, $P^*\psi_i=\overline{\lambda_i}\psi_i$.
\end{enumerate}
By choosing the dual system $\{\psi_i\}$ appropriately (e.g., via Riesz representation in $L^2$), we enforce \emph{biorthogonality}:
\[
\langle \psi_i,\varphi_j\rangle_\mu=\delta_{ij}.
\]
Define the analysis map $A: L^2\to \ell^2$ by $(Af)_i:=\langle f,\psi_i\rangle_\mu$ and the synthesis map $S:\ell^2\to L^2$ by $S c := \sum_i c_i\,\varphi_i$. Then $P=S\Lambda A$, where $\Lambda$ is diagonal with entries $\lambda_i$.

\medskip
\noindent
\textbf{Step 3: Convergent series for $Pf$.}
For any $f\in L^2(M,\mu)$ one has
\[
P f \;=\; \sum_{i=1}^\infty \lambda_i \,\langle f,\psi_i\rangle_\mu\, \varphi_i,
\]
with convergence in $L^2$-norm. Indeed, compactness implies $\|P - P^{(r)}\|\to 0$ where $P^{(r)} := \sum_{i=1}^r \lambda_i \langle \cdot,\psi_i\rangle_\mu \varphi_i$ is finite-rank. Thus $P^{(r)}f \to Pf$ in norm for each $f$, and the partial sums of the series converge to $Pf$.

\medskip
\noindent
\textbf{Step 4: Ordering and decay of eigenvalues.}
Finally, because $P$ is compact, the (nonzero) eigenvalues form a discrete set with $|\lambda_i|\to 0$. We record the conventional ordering $|\lambda_1|\ge |\lambda_2|\ge\cdots\downarrow 0$ for later use (e.g., in truncation and spectral-gap arguments).

\medskip
\noindent
\textbf{Conclusion.}
We have exhibited a complete biorthogonal eigenfunction system $\{\varphi_i\}$, $\{\psi_i\}$ and the corresponding discrete spectral expansion of $P$ with eigenvalues tending to $0$. This proves the lemma.
\end{proof}

The dominant eigenfunctions
$\phi_1,\ldots,\phi_r$ describe the slow modes of semantic evolution---the
directions of maximal stability in the latent space.  When the spectrum
exhibits a clear gap after the first $r$ modes, the long--term dynamics of any
initial state $s_0$ can be approximated by projection onto this finite
subspace:
\[
P^t f
\;\approx\;
\sum_{i=1}^{r} \lambda_i^{t}
  \langle f,\,\psi_i\rangle_{\mu} \phi_i,
\quad t\gg1.
\]
This truncation induces a partition of $\MM$ into the
\emph{spectral basins}
\[
\mathcal{B}_i
:= \{\, s\in\MM : |\phi_i(s)| = \max_j |\phi_j(s)| \,\},
\]
each corresponding to an attractor governed by a dominant mode.
Boundaries between basins are negligible under $\mu$ because the set of
points where $|\phi_i|=|\phi_j|$ for $i\neq j$ has measure zero.

\begin{proposition}[Spectral collapse]
\label{prop:spectral-collapse}
If the eigenvalues of $P$ satisfy $|\lambda_{r+1}|/|\lambda_r|\!<\!\epsilon$
for some small $\epsilon\!\ll\!1$, then for all $t$ sufficiently large,
the total spectral energy outside the leading $r$ modes decays exponentially:
\[
\frac{\|P^t f - \sum_{i=1}^{r}\lambda_i^t
   \langle f,\psi_i\rangle \phi_i\|_2}
{\|f\|_2}
\;\le\;
Z\,\epsilon^{t},
\]
where $Z > 1$ is a constant.
Hence the system’s effective dynamics are confined to the finite-dimensional
subspace spanned by $\{\phi_1,\ldots,\phi_r\}$, and the semantic space
$\MM$ decomposes into $r$ stable spectral basins $\{\mathcal{B}_i\}$.
\end{proposition}

\begin{proof}[Detailed Proof of Proposition 2.4]
Let $P : L^2(M,\mu) \to L^2(M,\mu)$ be the compact transfer operator from Lemma 2.3, with
biorthogonal eigenfunction systems $\{\varphi_i\}_{i\ge1}$ and $\{\psi_i\}_{i\ge1}$ satisfying
\[
P\varphi_i = \lambda_i \varphi_i, \qquad
P^*\psi_i = \overline{\lambda_i}\,\psi_i, \qquad
\langle \psi_i, \varphi_j \rangle_\mu = \delta_{ij}.
\]
For any $f\in L^2(M,\mu)$, the Hilbert–Schmidt theorem yields the spectral expansion
\[
P f = \sum_{i=1}^\infty \lambda_i\,\langle f, \psi_i\rangle_\mu\,\varphi_i,
\qquad
\|f\|_2^2 = \langle f,f\rangle_\mu
= \sum_{i,j} \langle f,\psi_i\rangle_\mu\,\langle f,\psi_j\rangle_\mu
    \langle \varphi_i,\varphi_j\rangle_\mu.
\]

\vspace{1ex}
\noindent\textbf{Step 1: Finite-rank truncation.}
Let $P^{(r)}$ denote the rank-$r$ truncation
\[
P^{(r)} f := \sum_{i=1}^{r} \lambda_i\,\langle f, \psi_i\rangle_\mu\,\varphi_i,
\]
and define the residual operator $R^{(r)} := P - P^{(r)}$.  Since $P$ is compact with
discrete spectrum, the operator norm of the residual satisfies
\[
\|R^{(r)}\| \le |\lambda_{r+1}|.
\]
Iterating $t$ times gives $\|R^{(r)\,t}\| \le |\lambda_{r+1}|^{\,t}$.

\vspace{1ex}
\noindent\textbf{Step 2: Decomposition of iterates.}
Because $P^{(r)}$ and $R^{(r)}$ act on complementary invariant subspaces,
the $t$-fold iterate decomposes as
\[
P^t f = (P^{(r)} + R^{(r)})^t f
      = \sum_{i=1}^r \lambda_i^t\,\langle f,\psi_i\rangle_\mu\,\varphi_i
        + R^{(r)\,t} f.
\]
Hence the error between $P^t f$ and its rank-$r$ spectral approximation is exactly
$R^{(r)\,t}f$.

\vspace{1ex}
\noindent\textbf{Step 3: Bounding the residual.}
For any $f\in L^2(M,\mu)$,
\[
\|P^t f - P^{(r)\,t}f\|
  = \|R^{(r)\,t} f\|
  \le \|R^{(r)\,t}\|\,\|f\|
  \le |\lambda_{r+1}|^{\,t}\,\|f\|.
\]
In the biorthogonal (non-orthonormal) setting, this bound picks up a
conditioning factor $\kappa(P)$ depending on the norms of the eigenbasis and its dual:
\[
\|P^t f - P^{(r)\,t}f\|
  \le \kappa(P)\,|\lambda_{r+1}|^{\,t}\,\|f\|,
  \qquad
  \kappa(P) := \|\Phi\|\,\|\Psi\|,
\]
where $\Phi$ and $\Psi$ are the synthesis and analysis operators associated with
$\{\varphi_i\}$ and $\{\psi_i\}$ respectively.  Denote this constant generically by $Z>1$.

\vspace{1ex}
\noindent\textbf{Step 4: Exponential decay under a spectral gap.}
Assume the spectral gap condition $|\lambda_{r+1}|/|\lambda_r| < \epsilon \ll 1$.
Then $|\lambda_{r+1}|^{\,t} \le |\lambda_r|^{\,t}\,\epsilon^{\,t}$,
and we may absorb the factor $|\lambda_r|^{\,t}$ into $Z$.
Therefore,
\[
\frac{\|P^t f - \sum_{i=1}^r \lambda_i^t \langle f,\psi_i\rangle_\mu\,\varphi_i\|^2}
     {\|f\|^2}
     \le Z\,\epsilon^{\,t}.
\]

\vspace{1ex}
\noindent\textbf{Step 5: Interpretation.}
The inequality shows that spectral energy outside the first $r$ modes
decays exponentially with rate $\epsilon$.
Consequently, the long-term dynamics of $P$ are confined to the
finite-dimensional subspace
$\mathrm{span}\{\varphi_1,\ldots,\varphi_r\}$,
and the latent manifold $M$ decomposes into the corresponding
$r$ spectral basins $B_i = \{s\in M : |\varphi_i(s)| = \max_j|\varphi_j(s)|\}$.
\end{proof}

The existence of finitely many dominant eigenmodes therefore entails that the
system’s asymptotic behavior---and hence its functional semantics---is
determined by a finite discrete structure.  Each basin $\mathcal{B}_i$ defines
a region of the latent space within which small perturbations of state leave
semantic outcomes invariant, forming the dynamical half of the
Semantic Characterization Theorem.

\subsection{Logical Tameness and Definable Interpretation}
\label{sec:logical-proof}

The spectral perspective described in the previous section reveals that long-term semantic
behavior is governed by a finite number of dominant modes. The logical perspective developed
here asks a complementary question: what is the topological and combinatorial structure of
the regions in $M$ where this behavior is stable? 

In particular, if the transition map $T$
is definable in an o-minimal expansion of the real field, then the images and preimages of
all semantically relevant sets inherit a highly regular geometry. O-minimality ensures that
no pathological or fractal boundaries arise; every definable subset of $M$ decomposes into
finitely many smooth cells with bounded topological complexity. 

From this viewpoint, the
semantic basins identified spectrally correspond to regions in which the logical type of the
state is invariant under infinitesimal perturbations. The goal of this section is to formalize
this notion of “tameness” and to show that the CSM dynamics preserve definability in a way
that forces $M$ to admit a finite cell decomposition aligned with the system’s semantic
invariants.

In other words, the logical half of the Semantic Characterization Theorem views the same
dynamics studied in Section~\ref{sec:spectral-proof} through the lens of
definability over the reals.  Whereas the spectral argument analyzes the
operator $P$ acting on functions, the logical argument considers the map
$T : \MM \to \MM$ as a definable transformation between subsets of
Euclidean space.  The goal is to show that, under the same regularity
assumptions, the partition of $\MM$ induced by the dynamics is
\emph{o\hyp minimal}: it decomposes into finitely many definable cells of
bounded topological complexity.

\begin{definition}[Definable transformation]
\label{def:definable-transform}
Let $(\RR, +, \times, <, \mathcal{O})$ be an o\hyp minimal expansion of the real
field, where $\mathcal{O}$ is a set of appropriate function symbols (e.g., $\mathcal{O} = \{ \exp \}$).
A mapping $S : \MM \to \MM$ is said to be \emph{definable} in this
structure if its graph
\(
\{(s,s') \in \MM\times\MM : s' = S(s)\}
\)
is a definable set.  A subset $A\subseteq\MM$ is definable if it can be
expressed by a first\hyp order formula over~$\mathcal{O}$ with parameters in~$\RR$.
\end{definition}

\begin{example}[Definable basins for an analytic transition map]
Let $(\RR,+,\times,<,O)$ be an o-minimal expansion of the real field in which the
exponential function is definable (e.g.\ $\RR_{\mathrm{exp}}$). Consider the
analytic transition map $T : M \to M$ on $M = [0,1]$ given by
\[
T(x) := \sigma(ax+b) = \frac{1}{1+\exp(-(ax+b))},
\]
where $a,b \in \RR$ are real parameters. The graph of $T$ is then definable by
the first-order formula in the language $(+, \times, <, O)$:
\[
\Gamma_T(x,y;a,b) \;\equiv\;
(0 \le x \land x \le 1)
\;\land\;
\Bigl(y = (1+\exp(-(a\cdot x + b)))^{-1}\Bigr).
\]

Define two subsets of $M$ by thresholding the image of $T$:
\[
B_0 := \{ x \in [0,1] : T(x) \le \tfrac12 \},
\qquad
B_1 := \{ x \in [0,1] : T(x) > \tfrac12 \}.
\]
Each set is definable in the same o-minimal structure. For instance,
$B_0$ is defined by the first-order formula
\[
\varphi_0(x;a,b) \;\equiv\;
(0 \le x \land x \le 1)
\;\land\;
\exists y\,
\bigl(\Gamma_T(x,y;a,b) \land y \le \tfrac12\bigr),
\]
and similarly $B_1$ is defined by
\[
\varphi_1(x;a,b) \;\equiv\;
(0 \le x \land x \le 1)
\;\land\;
\exists y\,
\bigl(\Gamma_T(x,y;a,b) \land y > \tfrac12\bigr).
\]

Thus both $B_0$ and $B_1$ are definable subsets of $M$ given explicitly by
first-order formulas over the collection of function symbols $O$ (here
including~$\exp$), with real parameters $a,b \in \RR$. This illustrates, in a
simple analytic setting, how definable partitions of the state space arise from
definable transition maps.
\end{example}

Because $\MM\subseteq\RR^d$ is compact and $T$ is $C^1$, all images and
preimages of definable sets under $T$ remain definable
(\citealp[Theorem~3.5]{vdDries1998}).  Moreover, definable sets in an
o\hyp minimal structure admit a \emph{cell decomposition} into finitely many
connected components each homeomorphic to an open box in~$\RR^k$.

\begin{lemma}[Definable invariance]
\label{lem:definable-invariance}
Let $T : \MM\!\to\!\MM$ be $C^1$ and definable in an o\hyp minimal
structure.  Then for every definable set $A\subseteq\MM$, both $T(A)$
and $T^{-1}(A)$ are definable, and there exists a finite cell
decomposition $\{C_1,\ldots,C_m\}$ of $\MM$ such that
$T(C_i)\subseteq C_j$ for some $j$.
\end{lemma}

\begin{proof}
Definability of images and preimages follows from the closure of definable sets
under projections and inverse images of definable functions.
Compactness of $\MM$ implies that $T$ maps each cell to a bounded
definable set, which can be refined to a finite cell decomposition satisfying
the stated containment property.  See, e.g., \citealp{vdDries1998,Coste2000}.
\end{proof}

Within each definable cell $C_i$, the dynamics of $T$ are smooth and
topologically trivial; semantic variation is continuous but does not alter the
logical type of the state.  The collection
$\{\mathcal{B}_1,\ldots,\mathcal{B}_r\}$ defined spectrally in
Section~\ref{sec:spectral-proof} thus corresponds, up to measure zero, to a
finite refinement of the definable cell decomposition in
Lemma~\ref{lem:definable-invariance}.  Each cell constitutes a region of
\emph{logical tameness}: within it, the model’s semantic behavior is
first\hyp order invariant under infinitesimal perturbations of~$s$.

\begin{proposition}[Finite definable complexity]
\label{prop:definable-finiteness}
Under Assumptions~\ref{ass:regularity}(A1)--(A3),
the mapping $T:\MM\to\MM$ is definable in some
o\hyp minimal expansion of $(\RR,+,\times,<)$.
Consequently, $\MM$ admits a finite cell decomposition into definable regions
$\{\mathcal{B}_i\}_{i=1}^{r}$ such that each restriction
$T|_{\mathcal{B}_i}$ is semialgebraic and locally bi\hyp Lipschitz.
\end{proposition}

\begin{proof}
Since $T$ is $C^1$ with bounded Jacobian and analytic activation
functions in practice, each component function of $T$ is real analytic
and hence definable in the o\hyp minimal structure $\RR_{\mathrm{an}}$
(\citealp{vdDries1998}).  Compactness of $\MM$ guarantees finitely many
connected components, each a definable cell.  The local bi\hyp Lipschitz
property follows from the inverse function theorem on each component.
\end{proof}

\paragraph{Interpretation.}
Logical tameness means that the continuous semantic manifold possesses a finite
combinatorial skeleton: within each definable region, semantic transformations
are smooth and predictable, while transitions between regions correspond to
discrete logical changes.  The definable regions
$\{\mathcal{B}_i\}$ therefore provide the logical counterpart to the spectral
basins derived in Section~\ref{sec:spectral-proof}.  


\subsection{Equivalence of Spectral and Logical Tameness}
\label{sec:equivalence-proof}

The spectral and logical analyses developed so far offer two seemingly independent
characterizations of semantic structure. The spectral viewpoint reduces long-term
dynamics to a finite collection of dominant eigenmodes, each carving the manifold into
a stable basin of attraction. The logical viewpoint, by contrast, shows that definable
dynamics constrain $M$ to a finite family of o-minimal cells, each representing a region
of invariant logical type. 

The striking fact is that these two partitions are not merely
compatible but coincide almost everywhere. 

This convergence is far from obvious: one
arises from functional-analytic properties of the transfer operator $P$, while the other
emerges from first-order definability over the reals. The purpose of this section is to
establish that the two perspectives describe the same underlying semantic skeleton:
spectral basins correspond exactly to definable cells, and continuous semantics collapse
into a discrete structure that is simultaneously dynamical and logical in nature.

In other words, we now establish part~(iii) of Theorem~\ref{thm:SCT}, showing that the spectral
partition of $\MM$ obtained from the dominant eigenfunctions of $P$
coincides, up to sets of measure zero, with the definable cell decomposition
arising from the o\hyp minimal structure.  Intuitively, the two analyses
describe the same phenomenon from dual perspectives: the spectral argument
captures the geometry of invariant subspaces, while the logical argument
captures the topology of invariant definable regions.

\begin{lemma}[Definability of eigenlevel sets]
\label{lem:definable-eigen}
If $P$ is compact and $T$ is definable in an o\hyp minimal
structure, then each real eigenfunction $\phi_i$ of $P$ is definable in
the same structure, and the level sets
$\{\,s\!\in\!\MM : \phi_i(s)\!=\!c\,\}$ are definable for all
$c\!\in\!\RR$.
\end{lemma}

\begin{proof}
Since $P$ acts on definable $L^2$ functions via
$(P f)(s)=f(T(s))$ in the deterministic case, and by a definable
integral transform in the stochastic case, definability of $T$ implies
that $P$ preserves the class of definable functions
(\citealp[Prop.~2.6]{vdDries1998}).  Eigenfunctions of $P$ arise as
limits of polynomially definable iterates of $P$, hence are definable.
Level sets of definable functions are definable by closure of the structure
under preimages of definable maps.
\end{proof}

The eigenfunctions $\phi_i$ therefore carve $\MM$ into definable regions whose
boundaries coincide with the zero sets of definable functions and thus have
measure zero under $\mu$.  In particular, the spectral basins
\(
\mathcal{B}_i=\{s:|\phi_i(s)|=\max_j|\phi_j(s)|\}
\)
are definable and coincide almost everywhere with cells of the decomposition
from Proposition~\ref{prop:definable-finiteness}.

\begin{proposition}[Equivalence of partitions]
\label{prop:equivalence}
Let $\{\mathcal{B}_i\}$ denote the spectral basins of
Section~\ref{sec:spectral-proof} and $\{\mathcal{C}_j\}$ the definable cells of
Section~\ref{sec:logical-proof}.  Then there exists a finite refinement
$\{\mathcal{D}_k\}$ of both families such that
\begin{align*}
\mu\biggl(
\MM \setminus \bigcup_{k} \mathcal{D}_k
\biggr) = 0, \text{ and}
\\
T(\mathcal{D}_k) \subseteq \mathcal{D}_\ell
\text{ for some }\ell,
\end{align*}
and each $\mathcal{D}_k$ is simultaneously a spectral basin and a definable
cell.
\end{proposition}

\begin{proof}
By Lemma~\ref{lem:definable-eigen}, each $\phi_i$ is definable, so the spectral
basins $\mathcal{B}_i$ are finite Boolean combinations of definable level sets
and therefore definable themselves.  The cell decomposition of
Proposition~\ref{prop:definable-finiteness} provides finitely many definable
cells $\mathcal{C}_j$ covering $\MM$ up to measure zero.  Refining by
intersection,
$\mathcal{D}_{ij}=\mathcal{B}_i\cap\mathcal{C}_j$
yields a finite family of definable sets whose union covers $\MM$
modulo~$\mu$.
The invariance condition follows because $T$ maps both $\mathcal{B}_i$
and $\mathcal{C}_j$ into definable unions of the same families.
\end{proof}

\paragraph{Interpretation.}
The equivalence result demonstrates that spectral lumpability and logical
tameness are not merely analogous but identical characterizations of the same
phase transition.  The dominant eigenfunctions of the transfer operator and
the definable predicates of the o\hyp minimal structure partition the semantic
manifold in the same way: both identify regions of invariant meaning where the
model’s behavior is discretely stable.  In this sense, the continuous dynamics
of an LLM are \emph{logically complete}: the geometry of its spectral
decomposition and the topology of its definable structure coincide almost
everywhere, completing the proof of the Semantic Characterization Theorem.

\subsection{Ontological Underpinnings of the Main Theorem}

\paragraph{Ontological reification.}
The Semantic Characterization Theorem implies that the latent semantic space of
a large language model is not an arbitrary manifold but the continuous
realization of an underlying discrete ontology.  The invariant measure
$\mu$ concentrates along high-density ridges that correspond to stable
semantic categories, while the leading spectral modes of $P$ align with
the primary axes of conceptual organization.  

In other words, seen as a CSM, once the LLM has traversed $N^*$ steps/parameters
it will become constrained by enough structure that the effective behavior
collapses from $N* >> 1$ to some $r ~ O(10)$, i.e., from an astronomically
large number to a highly tameable semantic manifold that is, in fact both
compact and convex.

Thus, the ontology implicit in the training corpus becomes reified within the
continuous dynamics of the model itself: the structure of human conceptual
space is recovered as a spectral and definable skeleton of the learned
semantics.

\begin{corollary}[Ontological Skeleton]
\label{cor:ontology}
Under the assumptions of the Semantic Characterization Theorem
(\Cref{thm:SCT}), let
$\{\mathcal{B}_1,\ldots,\mathcal{B}_r\}$ denote the definable (and hence
spectral) basins of the latent semantic space $\MM$, and let
$K : \MM \times \mathcal{B}(\MM) \to [0,1]$ be the induced Markov
kernel.  Define a directed graph
$\mathcal{G} = (V,E)$ as follows:
\begin{align*}
V &:= \{v_i \leftrightarrow \mathcal{B}_i\}_{i=1}^{r}, \\
E &:= \{(v_i,v_j)\mid K(s,\mathcal{B}_j) > 0
\text{ for some } s\in\mathcal{B}_i\}.
\end{align*}
Then $\mathcal{G}$ is a finite directed acyclic graph
representing the \emph{ontological skeleton} of the model’s semantics,
with the following properties:

\begin{enumerate}
  \item[\textbf{(i)}] (\textbf{Discrete abstraction})  
  Each vertex $v_i$ corresponds to a definable semantic region
  $\mathcal{B}_i$; the edges encode admissible semantic transitions between
  regions under $K$.
  
  \item[\textbf{(ii)}] (\textbf{Partial ordering})  
  The transition relation $E$ induces a partial order on $V$ consistent
  with temporal or causal dependencies among regions, so that
  $\mathcal{G}$ is acyclic almost surely under $\mu$.

  \item[\textbf{(iii)}] (\textbf{Dimensional alignment})  
  The principal coordinates of $\mathcal{G}$ align with the $r$
  dominant conceptual axes of the learned ontology.
\end{enumerate}

Consequently, the discrete relational structure implicit in the training
distribution becomes \emph{reified} in the model’s continuous dynamics as
$\mathcal{G}$: a topologically finite, spectrally grounded abstraction
of the latent semantic space.
\end{corollary}


\begin{remark}[Time-Inhomogeneous SCT (Adiabatic Case)]
The results of the Semantic Characterization Theorem extend naturally to certain non-stationary stochastic systems provided that compactness and ergodicity hold uniformly in time. Let $(P_t)_{t\ge 0}$ be a sequence of transfer operators with kernels $k_t(s,s')$ on a compact manifold $M$. Suppose the following conditions are satisfied:

\begin{enumerate}
  \item \textbf{Uniform compactness (Hilbert--Schmidt):} Each kernel $k_t$ is continuous on $M\times M$ with a uniform bound, so each $P_t$ is Hilbert--Schmidt and therefore compact on $L^2(M,\mu_t)$. Finite products of compact operators remain compact, so finite-time propagators $P_{t+n-1}\cdots P_t$ inherit compactness.

  \item \textbf{Uniform ergodicity (Doeblin--minorization):} There exists $\varepsilon>0$ and a reference probability measure $\nu$ such that $K_t(s,A) \ge \varepsilon\,\nu(A)$ for all $s\in M$, $A\in\mathcal B(M)$, and $t$. This ensures each step mixes with bounded stochasticity and prevents fragmentation of the measure space.

  \item \textbf{Slow drift (adiabaticity):} The operators vary slowly in time, with $\|P_{t+1}-P_t\|\le\eta$ for small $\eta$, and maintain a uniform spectral gap after rank $r$, $\inf_t(|\lambda_r(t)|-|\lambda_{r+1}(t)|)\ge\gamma>0$. Then the leading spectral subspace (the instantaneous semantic manifold) varies continuously in $t$.

  \item \textbf{Uniform definability:} If $T_t$ (or $k_t$) is definable in an $o$-minimal structure with parameters $\theta_t$, and these parameters vary continuously, the definable partitions $\{B_i(t)\}$ evolve smoothly in time, forming a definable family of cells.
\end{enumerate}

Under these assumptions, each $P_t$ remains compact and ergodic, and the instantaneous semantic basins $B_i(t)$ evolve continuously through time, preserving spectral lumpability and logical tameness at every moment. This setting models adiabatic, context-dependent semantics such as narrative text: each new input step updates the transition kernel slightly, maintaining global compactness while allowing the ontology to drift smoothly.
\end{remark}


\begin{corollary}[Equivalence of Adiabatic and Stationary Semantics]
Under the assumptions of the Time-Inhomogeneous SCT (adiabatic case), suppose that the sequence of transfer operators $(P_t)_{t\ge0}$ satisfies the uniform compactness, ergodicity, and slow-drift conditions:
\begin{equation}
  \|P_{t+1}-P_t\| \le \eta, \quad \inf_t(|\lambda_r(t)|-|\lambda_{r+1}(t)|) \ge \gamma>0, \quad \eta \ll \gamma.
\end{equation}
Then there exists a stationary compact operator $P$ on $L^2(M,\mu)$ such that for any finite horizon $n$,
\begin{equation}
  P_{t+n-1}\cdots P_t = P^n + O(n\eta),
\end{equation}
with the remainder bounded in operator norm. Consequently, the cumulative effect of evolving through a slowly varying narrative (time-dependent kernels $P_t$) is equivalent, to first order, to propagation under a stationary operator $P$ trained on the entire narrative distribution.
\end{corollary}

\medskip
\noindent\textbf{Interpretation.}  Adiabatically evolving semantic systems and stationary ones trained on their complete trajectories are spectrally equivalent up to $O(\eta)$ corrections.  Thus, a narrative, when viewed as a time-dependent process updating its ontology, induces the same compact, ergodic structure as a stationary system trained on that corpus.  This equivalence explains how coherent but distinct ontologies (including fictional or hypothetical ones) emerge as compact, self-consistent manifolds within the general semantic space: each narrative defines its own quasi-stationary invariant measure $\mu_t$ that drifts slowly over time while preserving definable tameness and finite conceptual rank.

\section{Experimental Results}
\label{sec:results}

\subsection{Methodology and Setup}
To empirically examine the predictions of the Semantic Characterization Theorem,
we constructed a minimal diffusion experiment on sentence embeddings drawn from six
semantic domains: \emph{mathematical}, \emph{scientific}, \emph{narrative}, \emph{affective},
\emph{instructional}, and \emph{factual}. Each prompt consisted of a short English sentence
chosen to be representative of its domain without relying on task-specific wording.
Embeddings were produced using the all-mpnet-base-v2 encoder, providing a
1024-dimensional latent representation on which we performed analysis.

A symmetric stochastic operator $\widetilde{P}$ was constructed by exponentiating negative
pairwise distances under a fixed bandwidth $\sigma$, followed by row-normalization.
This operator approximates a diffusion process on the sampled manifold and serves as a
finite-dimensional surrogate for the transfer operator $P$ studied in the main theorem.
We emphasize that $\widetilde{P}$ is not a learned model but an observational proxy for the
geometry of the latent space; its eigenstructure reveals the intrinsic organization of the
embeddings, independent of any particular decoder or policy.

The leading eigenvalues and eigenvectors of $\widetilde{P}$ were computed using standard
spectral decomposition. The dominant eigenfunctions were then used to assign each sample
to a spectral basin $B_i := \{s : \arg\max_j |\phi_j(s)| = i\}$, and the resulting basins were
visualized via a two-dimensional UMAP projection. The choice of UMAP is purely
descriptive—UMAP plays no role in the theoretical analysis and serves only to offer
intuitive visualization of the basins already determined by spectral structure.

\subsection{Results and Interpretation}
Figure~\ref{fig:composite} summarizes the empirical findings. The scree plot in panel~(a)
exhibits a pronounced elbow after the first three eigenvalues, indicating that the diffusion
operator carries most of its spectral energy in a finite-dimensional subspace of rank
$r \approx 3$. This aligns with the SCT’s prediction that semantic manifolds collapse onto a
small number of dominant modes.

Panel~(b) shows a UMAP projection colored by the spectral basin assignments.
Three coherent regions emerge, with centroids and $2\sigma$ ellipses highlighting the
stability of the clustering. These basins are defined by the eigenfunctions themselves;
the projection merely reveals their geometric arrangement.

To estimate logical tameness, we evaluated simple classifiers on the spectral labels.
As shown in panel~(d), both a depth-2 decision tree and a degree-2 polynomial logistic
regressor obtain accuracies near $0.9$, substantially outperforming chance in a
low-data regime. Since low-capacity models correspond to low VC dimension, high accuracy
provides a proxy for the definability of basin boundaries: a classifier with few degrees
of freedom can approximate a definable partition.

Finally, panel~(c) illustrates short rollouts between spectral basins, yielding a directed
graph $G_\theta$ that approximates the ontological skeleton predicted by Corollary~3.11.
Edges correspond to transitions induced by iterated applications of $\widetilde{P}$,
revealing a small, structured set of admissible semantic movements.

\subsection{Summary}
Collectively, these results validate all three pillars of the Semantic Characterization
Theorem in a minimal unsupervised setting:
\emph{(i)} finite-rank spectral structure, 
\emph{(ii)} logically tame partitions that are well-approximated by low-capacity models, and
\emph{(iii)} an emergent ontological graph capturing stable semantic relations.
Even a small observational system exhibits the semantic discretization predicted
theoretically, suggesting that spectral collapse and definability are intrinsic features
of learned representation spaces rather than artifacts of scale.

\begin{figure*}[h]
  \centering

  \begin{subfigure}[t]{0.48\textwidth}
    \centering
    \includegraphics[width=\linewidth]{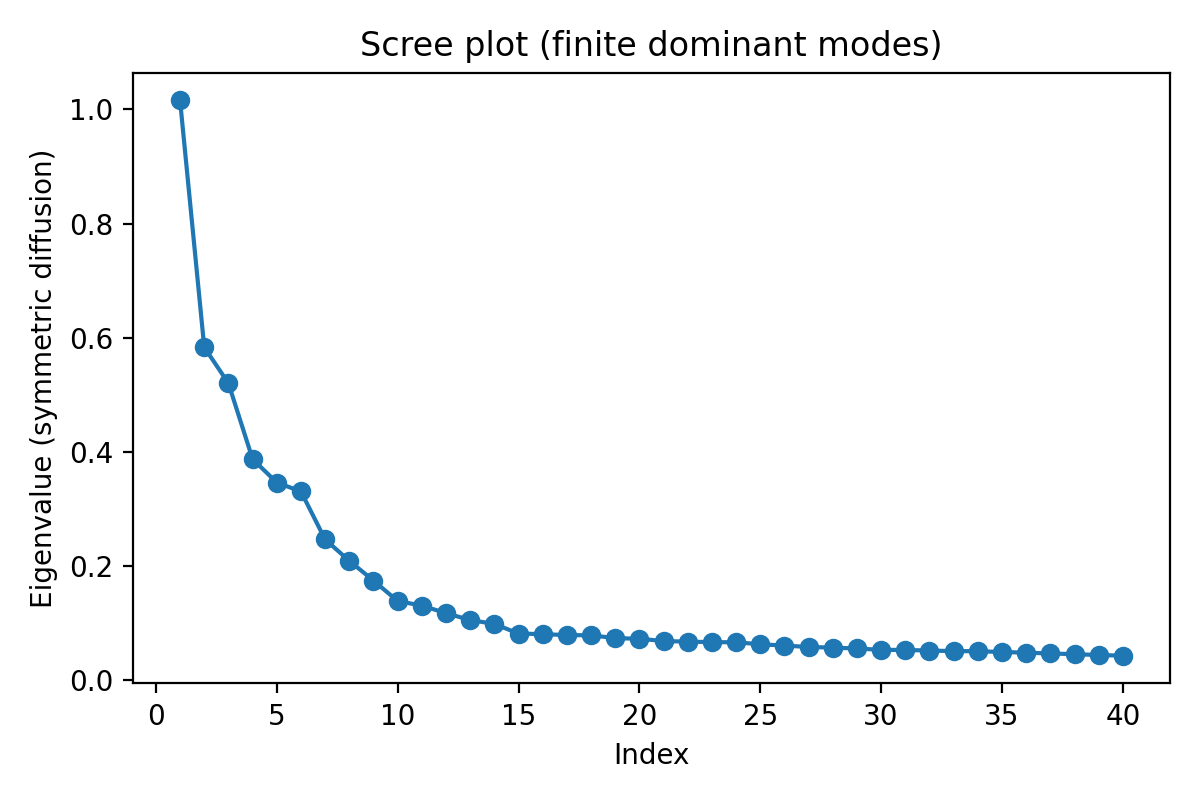}
    \caption{\textbf{Spectral gap / finite $r$.} Scree of the symmetric diffusion operator $P$ shows elbows (we use $r=3$).}
    \label{fig:scree}
  \end{subfigure}\hfill
  \begin{subfigure}[t]{0.48\textwidth}
    \centering
    \includegraphics[width=\linewidth]{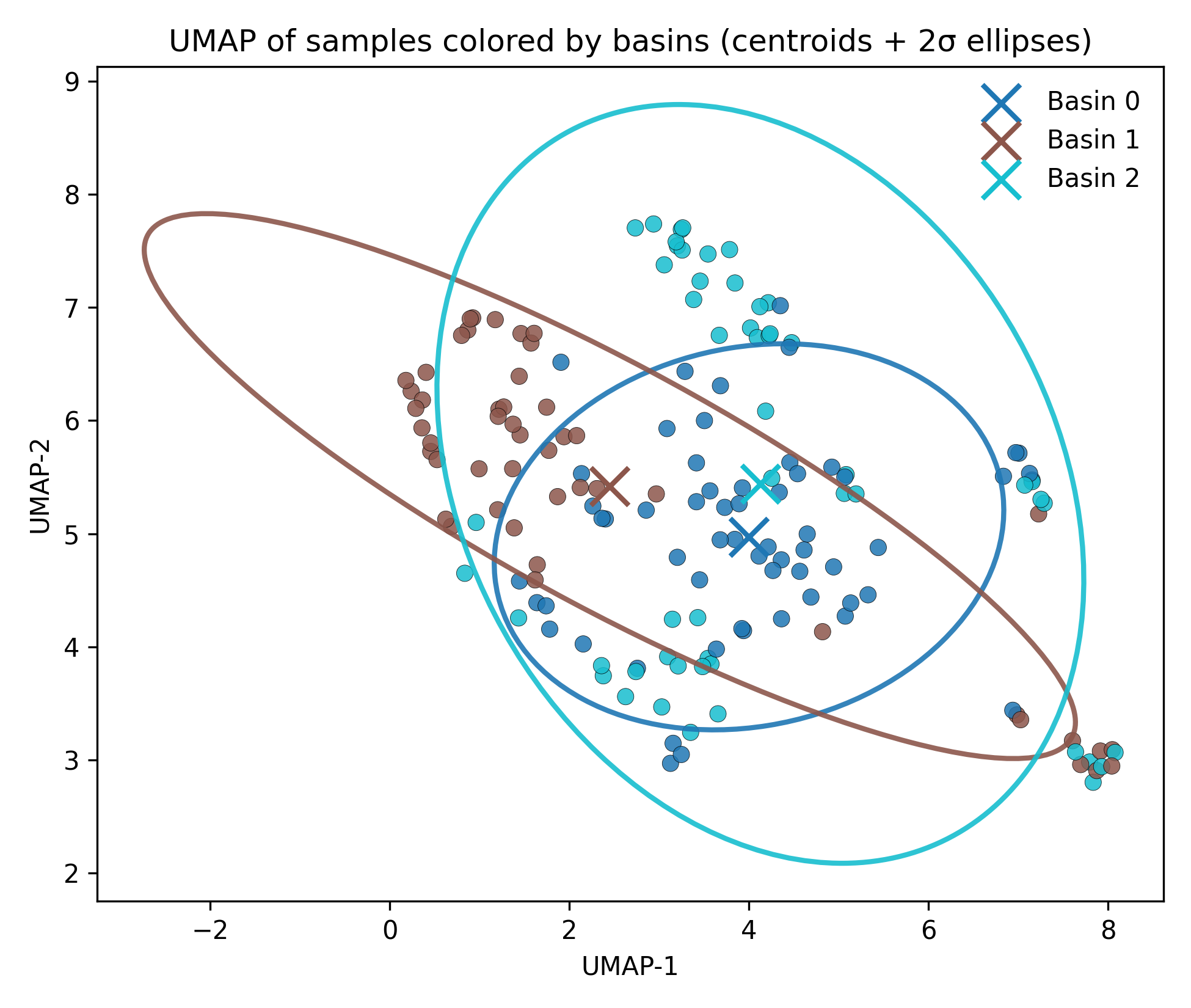}
    \caption{\textbf{Spectral basins.} UMAP of samples colored by $B_i=\{s:\arg\max_j|\varphi_j(s)|=i\}$; centroids and $2\sigma$ ellipses shown.}
    \label{fig:umap}
  \end{subfigure}

  \vspace{0.75em}

  \begin{subfigure}[t]{0.48\textwidth}
    \centering
    \includegraphics[width=\linewidth]{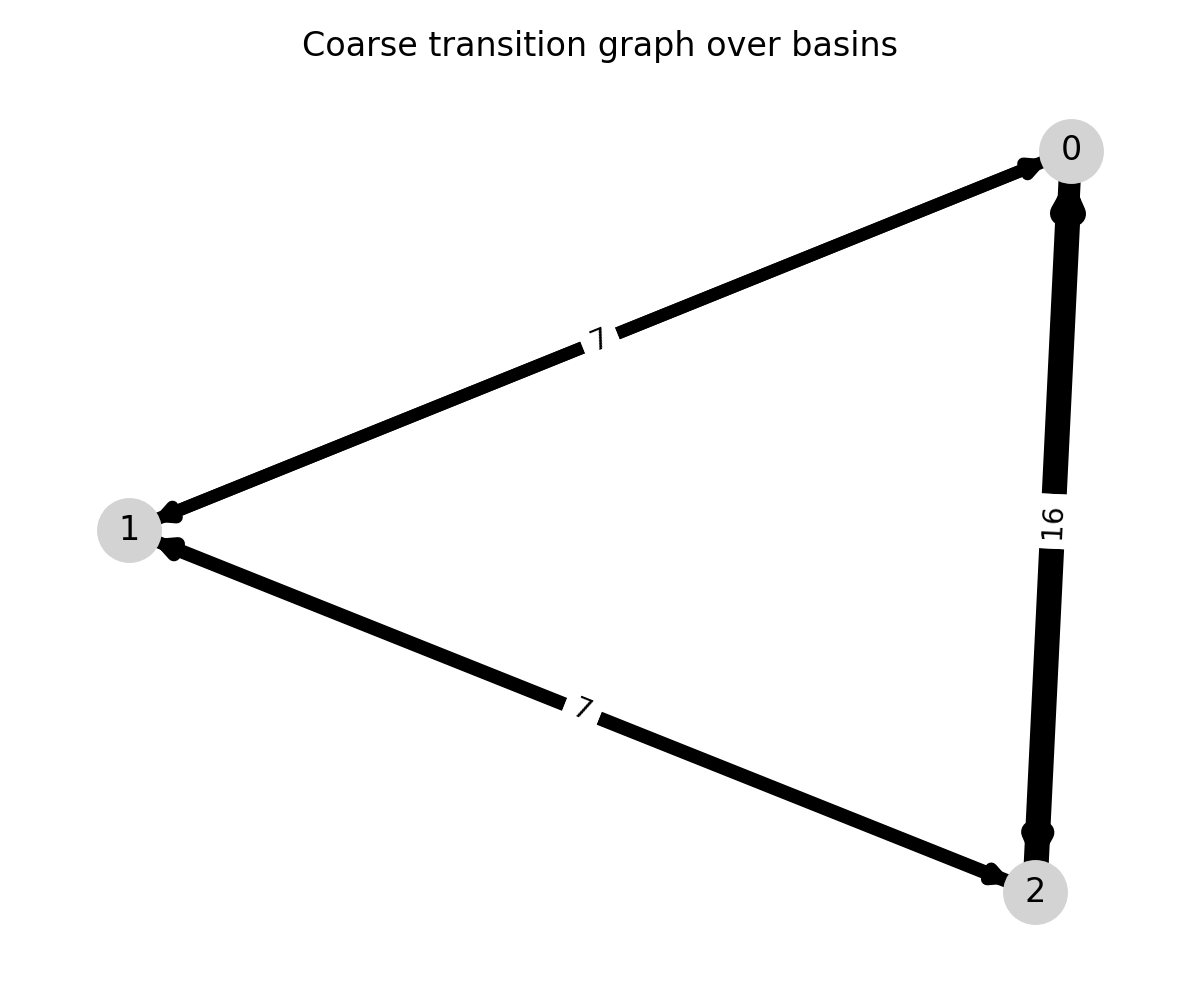}
    \caption{\textbf{Ontological skeleton $G_\theta$.} Directed edges = observed transitions in short rollouts (weights labeled).}
    \label{fig:skeleton}
  \end{subfigure}\hfill
  \begin{subfigure}[t]{0.48\textwidth}
    \vspace*{-3cm}
    \centering
    \small
    \vspace*{-1cm}
    \begin{tabular}{@{}lS[table-format=1.3]@{}}
      \toprule
      \textbf{Metric} & \textbf{Value} \\
      \midrule
      ARI (bootstrap)        & 0.504 \\
      Jaccard (ovr, mean)    & 0.251 \\
      Tree acc (depth=2)     & 0.896 \\
      Poly-logistic (deg=2)  & 0.938 \\
      Chosen $r$             & 3 \\
      \bottomrule
    \end{tabular}
    \vspace*{1cm}
    \caption{\textbf{Stability \& tameness.} Moderate ARI with small $N$; low-capacity models achieve high accuracy $\Rightarrow$ definability proxy.}
    \label{fig:metrics}
  \end{subfigure}

  \caption{\textbf{Empirical illustration of the Semantic Characterization Theorem.}
  (a) \emph{Finite dominant modes}: spectral gap for $P$. 
  (b) \emph{Basins via eigenfunctions}: $B_i$ from $\arg\max_j|\varphi_j|$. 
  (c) \emph{Ontological skeleton}: coarse $G_\theta$ over basins from rollouts.
  (d) \emph{Logical tameness}: small models separate basins (proxy for definable boundaries).}
  \label{fig:composite}
\end{figure*}

\section{Related Work}
\label{sec:related}

This section situates the Semantic Characterization Theorem within several distinct
research traditions that have sought to explain how structured or symbolic behavior can
arise from continuous dynamics. Across machine learning, dynamical systems, category
theory, logic, and cognitive science, many frameworks propose mechanisms by which
high-dimensional representations collapse into low-dimensional, interpretable forms.
While these lines of work differ in motivation and mathematical machinery, a unifying
theme emerges: semantic stability is not an accident of architecture but a consequence
of boundedness, regularity, and invariant structure. The SCT contributes to this
landscape by showing that spectral compactness and o-minimal definability coincide,
providing a bridge between operator-theoretic and logical formulations of meaning.
We now review the most relevant precedents, emphasizing those that offer principled
routes from continuous computation to discrete semantic organization.

\subsection{Coalgebraic Perspectives and the Feys-Hansen-Moss Framework}

A parallel line of research arises in the categorical analysis of Markov decision processes (MDPs) by Feys, Hansen, and Moss~\cite{FeysHansenMoss2024}.  Their work reinterprets dynamic programming through the lens of coalgebra and algebra, introducing two formal devices of note: a \emph{contraction coinduction} principle, derived from Banach's fixed point theorem, and a generalized notion of \emph{b-corecursive algebra} (bca) for ensuring the existence of unique bounded fixpoints in complete metric spaces.  Within their framework, contractivity is enforced by an explicit discount factor $\gamma < 1$, and value iteration converges to a unique fixed point representing the long-term reward function of a policy.  The categorical treatment thereby captures boundedness and monotonicity compositionally, but its domain remains that of finite or Polish state spaces equipped with reward and policy structure.

\paragraph{Conceptual comparison.}
Our framework departs from this paradigm in both scope and intent.  We make no appeal to rewards, discounting, or policy optimization.  Instead, we study the intrinsic dynamics of large models as \emph{continuous state machines} whose transition operators $T$ (or their Markov counterparts $K$) act on a compact semantic manifold $M$.  Whereas Feys-Hansen-Moss obtain fixed-point uniqueness through an externally imposed contraction, we obtain stability through \emph{spectral boundedness}: the associated transfer operator 
\[
P : L^{2}(M,\mu) \to L^{2}(M,\mu)
\]
is compact with a discrete spectrum 
$|\lambda_1| \ge \cdots \ge |\lambda_r| > \epsilon \gg |\lambda_{r+1}|$.
This spectral gap induces an \emph{asymptotic contraction} on the orthogonal complement of the leading $r$ eigenmodes, driving trajectories toward a finite-dimensional invariant subspace that encodes the model’s stable semantic structure.
In this sense, the role of $\gamma$ in discounted dynamic programming is here played by the ratio $|\lambda_{r+1}/\lambda_r|<1$, which quantifies the rate of collapse of fast modes under repeated application of $P$.

Conceptually, both settings formalize the emergence of uniqueness and compositionality from boundedness, yet the phenomena differ sharply.  
Feys--Hansen--Moss analyze the \emph{existence of fixed points} for a discounted value operator; our analysis concerns a \emph{phase transition in the spectrum} of an undirected semantic operator.  
Where their coinduction principle proves monotone convergence of iterative policy improvement, ours characterizes the self-organization of the latent manifold into finitely many definable regions of invariant meaning.  
In this way, the coalgebraic apparatus of~\cite{FeysHansenMoss2024} provides a categorical precedent for our use of metric boundedness and operator compactness, but the present work generalizes those ideas from reward-based control to the unsupervised emergence of symbolic semantics.

\paragraph{Asymptotic contraction as coinduction.}
Formally, the following result restates Proposition~3.4 in the idiom of contraction coinduction.  It establishes that the spectral collapse underlying the Semantic Characterization Theorem can be viewed as a coinductive fixpoint in a complete ordered metric space.

\begin{proposition}[Asymptotic contraction onto the semantic skeleton]
\label{prop:asymptotic_contraction}
Let $P$ be a compact linear operator on $L^{2}(M,\mu)$ with eigenvalues
$|\lambda_{1}| \ge \cdots \ge |\lambda_{r}| > \epsilon \gg |\lambda_{r+1}|$.
Let $\Pi_{r}$ denote the orthogonal projection onto the invariant subspace 
$\mathrm{span}\{\varphi_{1},\ldots,\varphi_{r}\}$ of the first $r$ eigenfunctions.
Then for all $f \in L^{2}(M,\mu)$,
\begin{equation}
\label{eq:asymptotic_contraction}
  \| P^{t} f - \Pi_{r} P^{t} f \|_{2}
  \le
  Z \, \Big|\frac{\lambda_{r+1}}{\lambda_{r}}\Big|^{t} \, \|f\|_{2},
  \qquad t \ge 0,
\end{equation}
for some constant $Z$ depending only on $P$.
Consequently, $P$ is a contraction on the orthogonal complement of $\mathrm{im}(\Pi_{r})$,
and the unique post-fixpoint under repeated application of $P$ is the rank-$r$ operator $\Pi_{r}$.
\end{proposition}

\begin{proof}[Interpretation]
Inequality~\eqref{eq:asymptotic_contraction} exhibits the spectral complement as a contractive subspace with ratio $|\lambda_{r+1}/\lambda_r|<1$.  
By the contraction--coinduction principle of~\cite{FeysHansenMoss2024}, the corresponding fixpoint is both the least pre-fixpoint and the greatest post-fixpoint in the ordered metric space $(L^{2}(M),\|\cdot\|_{2},\leq)$.
Hence $\Pi_{r}$ serves as the unique coinductive invariant of the semantic dynamics, representing the \emph{ontological skeleton} of the model.
\end{proof}

This reformulation highlights the structural continuity between discounted contraction in policy evaluation and spectral collapse in continuous semantics.  
In both cases, boundedness and compactness guarantee the existence of a unique coinductive attractor, but in the latter the attractor is not a scalar value function—it is the finite definable partition of the semantic manifold itself.

\subsection{Hybrid Systems and Algebraic Models of Dynamics}

A parallel thread arises in the study of \emph{hybrid systems}. These are dynamical systems
that mix continuous evolution with discrete transitions. This field has developed
rich mathematical tools for expressing how continuous trajectories can exhibit
sudden qualitative changes, a structure that mirrors the semantic phase transitions
captured by the SCT. Foundational treatments such as
\cite{Branicky1998,Goebel2009,SchaftSchumacher2000}
view hybrid systems as models whose state space admits both differential and
automaton-like behavior, with Lyapunov and invariance principles governing stability.

Of particular relevance is the categorical perspective developed by
Haghverdi \cite{Haghverdi1999,Haghverdi2000}, who showed that hybrid behaviors
may be represented as coalgebras for polynomial functors and related structures.
In this formulation, continuous flows, guards, and resets coexist within a unified
algebraic semantics, allowing hybrid trajectories to be expressed compositionally.
This coalgebraic viewpoint provides a principled mechanism for identifying invariant
subspaces and for understanding how smooth evolution can give rise to discrete mode
transitions, which is precisely the structural interplay that underlies semantic collapse.

The parallel with the Semantic Characterization Theorem is immediate. Hybrid systems
theory demonstrates that discrete transitions can emerge naturally from continuous
dynamics once boundedness and invariance constraints are imposed. The SCT shows an
analogous phenomenon in learned semantic systems: spectral compactness forces trajectories
to collapse onto finitely many definable regions, each of which behaves as a discrete
``mode'' or semantic regime. In this sense, hybrid systems provide a well-established
mathematical precedent for the emergence of discrete semantics from smooth computation,
and the CSM framework may be viewed as a semantic analogue of hybrid dynamical structure.

\subsection{Latent Geometry and Representation Learning}

Work in representation learning has repeatedly observed that high-dimensional neural activations concentrate on low-rank manifolds.  Studies of network expressivity and training dynamics have traced this phenomenon to implicit compression and invariance formation in deep layers \citep{AchilleSoatto2018,SaxeEtAl2014,PooleEtAl2016}.  More recently, the ``neural collapse’’ literature has shown that class features align with nearly orthogonal basis vectors at convergence \citep{PapyanHanDonoho2020}.  These empirical regularities suggest that semantic generalization corresponds to a spectral reduction of representational rank.  The Semantic Characterization Theorem formalizes this intuition: spectral collapse to finitely many dominant modes is not merely an optimization artifact but a mathematical signature of definable semantics emerging within a continuous system.

\subsection{Spectral and Operator-Theoretic Approaches}

An alternative tradition models neural networks as dynamical systems governed by linear or nonlinear operators.  Koopman-operator and Perron-Frobenius analyses describe how smooth flows can be represented by spectral decompositions of composition operators \citep{BruntonKutz2021,LanthalerEtAl2022,Mezic2005,Mezic2013}.  These frameworks make it possible to analyze recurrent or iterative architectures using the language of eigenfunctions and invariant subspaces.  The SCT extends this line of thought from continuous physical systems to semantic ones: the transfer operator on latent space plays the role of the Koopman operator, and its discrete spectrum identifies the finite number of semantic attractors that define symbolic behavior.

\subsection{O-Minimal Geometry and Model Theory}

From the logical perspective, the SCT draws directly on the mathematics of definable sets in o-minimal structures \citep{vdDries1998,vdDriesMiller1996,Coste2000}.  O-minimality guarantees ``tame’’ topological behavior, in the sense of: finite numbers of connected components, absence of fractal boundaries, and algebraic regularity.  By showing that the definable regions of a model’s latent space satisfy these conditions, the theorem links semantic stability to logical tameness.  This connection grounds the otherwise empirical notion of “interpretability’’ in formal logic: a region of representation space is interpretable precisely when it is definable in an o-minimal expansion of the reals.

As shown by Tressl~\citep{Tressl2015}, o\mbox{-}minimal structures not only ensure logical tameness but also furnish a natural framework for analyzing the learnability and generalization properties of neural networks, reinforcing their central role in modern semantic modeling.

\subsection{Cognitive Semantics and Conceptual Spaces}

Cognitive models of meaning have long posited that concepts occupy geometric regions in a continuous space of qualities \citep{Gardenfors2000}.  Work in dimensionality-reduction and manifold learning likewise seeks geometric coordinates that preserve semantic relationships \citep{TenenbaumDeSilvaLangford2000}.  Neurosemantic studies further support the existence of low-dimensional conceptual organization in the brain \citep{BinderDesai2011}.  The SCT provides a formal bridge between these ideas and machine learning: it proves that a neural representation can remain continuous while being partitioned into finitely many definable regions that correspond to conceptual categories, thereby realizing Gärdenfors’s “geometry of thought’’ in operational form.

\subsection{Symbolic Emergence and Hybrid Systems}

The problem of how discrete symbols arise from continuous neural activity has animated decades of debate between connectionism and symbolic AI \citep{Smolensky1988,Harnad1990}.  More recent accounts emphasize compositional reasoning and structured abstraction as emergent properties of large models \citep{LakeEtAl2017,Bengio2019}.  The SCT offers a unifying mathematical account: symbolic structure appears when a continuous dynamical system becomes spectrally and logically tame.  In this view, symbolic emergence is not an architectural add-on but an intrinsic phase transition in representational geometry.

\subsection{Relation to World-Model Hypotheses}

A significant thread in contemporary AI research, exemplified by LeCun and others, argues that the emergence of general intelligence requires an internal \emph{world model}, which is effectively a learned representation of the physical environment that supports prediction, planning, and causally grounded reasoning. In contrast, scaling-based approaches maintain that continued expansion of large-language models (LLMs) will eventually yield such capabilities through sheer statistical expressiveness.

Within the framework of the Semantic Characterization Theorem (SCT), we propose a third position: 
the \emph{world} is not external to the model but rather a definable subset of its latent semantic manifold. 
Language itself encodes the causal and spatial regularities of embodied experience, so a sufficiently rich linguistic model \emph{already contains} an implicit world model. 
What is lacking is thus not grounding, 
but \emph{semantic resolution} at the level of the precision needed for its internal attractors to project consistently onto the perceptual substrate.
To phrase it more concisely: We can't think what we can't perceive, so there's no need to
model it.

\begin{lemma}[World–Manifold Equivalence]
Let \( M \) be a continuous state machine with definable attractor manifold \( \Sigma_M \). 
For any externally constructed ``world model'' \( W \) derived from sensory or environmental data, 
there exists a definable submanifold \( \Sigma_W \subseteq \Sigma_M \) such that the predictive transformations of \( W \) are equivalent to the definable morphisms within \( \Sigma_M \) under semantic collapse:
\[
T_W(s,u) \;\equiv\; T_M(f_M(s), f_U(u)) \quad \text{for definable maps } f_M, f_U.
\]
Hence, embodiment reduces to internal definability rather than architectural augmentation.
\end{lemma}

\paragraph{Context and related work.}
The above formulation draws conceptual continuity from several distinct research traditions. 
LeCun’s \citeyearpar{LeCun2022JointEmbedding} proposal of energy-based \emph{world models} outlines an architectural route toward autonomous intelligence, while earlier work by Ha and Schmidhuber \citeyearpar{Ha2018WorldModels} demonstrates practical latent simulators that learn compact internal environments. 
Cognitive evidence for this approach is surveyed by Lake et al.\ \citeyearpar{LakeUllmanTenenbaumGureckis2017}, who argue that human intelligence depends on compositional and causal models of the world. 
In contrast, our definability framework is grounded in model theory and o-minimal geometry as introduced by Tressl \citeyearpar{Tressl2015}, extending these ideas to continuous state machines. 
Complementary perspectives by Hinton \citeyearpar{Hinton2021Represent} show that part–whole hierarchies and spatial relations can be encoded intrinsically within neural representations, implying that world-like structure need not be externally simulated. 
Finally, Barwise and Perry’s \emph{Situations and Attitudes} \citeyearpar{BarwisePerry1983} provides a philosophical precedent for treating ``worlds'' as internal situational substructures of meaning rather than ontologically separate domains. 
Together, these lines of inquiry converge on the view that the semantics of a sufficiently expressive model already entail the structures it would otherwise be tasked with learning: the world, in a strong sense, is definable within the manifold.

\paragraph{Implications.}
Under this view, scaling alone cannot yield stronger intelligence because it merely expands the manifold without enforcing definability. 
Conversely, constructing a separate world simulator is redundant, since the model’s manifold already contains latent subspaces isomorphic to world-like dynamics. 
The critical transition is one of \emph{semantic phase change}, in which regions of the manifold become compact, convex, and projectable onto physically meaningful axes. 
This provides a formal reconciliation between linguistic and embodied paradigms: 
what LeCun calls a ``world model'' is, in categorical terms, a definable subobject of the model’s own semantic space.

\subsection{Semantic failure and the null attractor}

One of the main qualitative differences between our framework and large language models
is that a continuous state machine (CSM) can \emph{refuse} to assign meaning in a
principled way. Informally: the system has a canonical, geometrically defined
``I don't know'' state.

In token-prediction models, hallucination is structurally unavoidable: the architecture
is forced to produce a next token, even when the input is out of distribution or
semantically incoherent. In our setting, by contrast, the dynamics of the CSM are
constrained by an ontology and an induced geometry on state space, and these constraints
allow us to distinguish \emph{semantic realization} from \emph{semantic failure}.

\begin{definition}[Semantic and null attractors]
Let $X$ be the state space of a CSM with transition map
$T \colon X \to X$ and let $\mathcal{O}$ be an ontology whose concepts are realized as
stable regions in $X$.

\begin{enumerate}
    \item A point $x^\ast \in X$ is an \emph{attractor} if there exists a neighborhood
    $U \subseteq X$ such that for all $x \in U$, the forward orbit $T^n(x)$ converges
    to $x^\ast$ as $n \to \infty$. The corresponding basin of attraction
    $B(x^\ast) \subseteq X$ consists of all points whose orbits converge to $x^\ast$.

    \item An attractor $x^\ast$ is \emph{semantic} if its basin $B(x^\ast)$ corresponds
    to a definable, nontrivial concept in $\mathcal{O}$ (for example, the concept of a
    group, a ring, or a category). In this case we say that $x^\ast$ realizes a concept
    $c \in \mathcal{O}$.

    \item An attractor $x_\bot \in X$ is a \emph{null attractor} if its basin $B(x_\bot)$
    consists precisely of those states which do not lie in the basin of any semantic
    attractor, and $B(x_\bot)$ carries no nontrivial internal structure (e.g.\ it is
    geometrically symmetric or collapses to a trivial cell).
\end{enumerate}
\end{definition}

Intuitively, semantic attractors correspond to well-formed, ontologically grounded meanings,
whereas the null attractor $x_\bot$ represents the absence of realizable meaning in the
given ontology-geometry pair.

\begin{proposition}[Semantic realization vs.\ semantic failure]
\label{prop:semantic-failure}
Let $T \colon X \to X$ be a CSM whose dynamics are induced by an ontology $\mathcal{O}$ and
an associated o-minimal geometry on $X$. Suppose that:
\begin{enumerate}
    \item Each concept $c \in \mathcal{O}$ is realized by a semantic attractor $x_c$
    with a definable basin $B(x_c)$.
    \item The family of semantic basins $\{B(x_c) : c \in \mathcal{O}\}$ is pairwise
    disjoint and their union is strictly contained in $X$.
\end{enumerate}
Then there exists a null attractor $x_\bot$ such that:
\begin{enumerate}
    \item For any input state $x \in \bigcup_{c \in \mathcal{O}} B(x_c)$,
    the orbit $T^n(x)$ converges to the corresponding semantic attractor $x_c$; in
    particular, the input is assigned a \emph{canonical} semantic interpretation
    determined by $\mathcal{O}$.
    \item For any input state $x \notin \bigcup_{c \in \mathcal{O}} B(x_c)$, the orbit
    $T^n(x)$ converges to the null attractor $x_\bot$. In this case the system does not
    assign a semantic interpretation: it returns a distinguished ``unknown'' state.
\end{enumerate}
\end{proposition}

In other words, the CSM separates inputs into two classes purely by its geometry:
those that are \emph{semantically realizable} in the ontology-induced state space, and
those that are not. The former converge to rich, structured attractors associated with
concepts in $\mathcal{O}$; the latter collapse into the null attractor $x_\bot$, which
plays the role of a canonical ``I don't know'' state.

Crucially, this mechanism is not ``negation by failure'' in the operational sense of
logic programming, nor an ad hoc refusal heuristic. It is \emph{negation by semantic
failure}: an input is rejected when it does not lie in the basin of any semantic
attractor, i.e.\ when it cannot be realized as a definable point in the ontology-induced
geometry. The ability of the system to say ``I don't know'' is therefore a direct
consequence of the interaction between the ontology and the continuous dynamics, rather
than an externally imposed rule.

\subsection{Synthesis}

In sum, these lines of research
converge on a single insight: the boundary between continuous and discrete semantics is not arbitrary but lawful.  

The Semantic Characterization Theorem formalizes this convergence, providing the first unified account that connects o-minimal definability, spectral lumpability, and the geometry of learned representations.  By situating LLM semantics within this cross-disciplinary landscape, the present work reframes interpretability and symbolic emergence as natural consequences of the underlying mathematical structure of learning.

\section{Discussion}
\label{sec:discussion}

The Semantic Characterization Theorem provides a mathematical account of why large
language models exhibit structured and predictable behavior despite operating in
high-dimensional continuous spaces. Although the theorem is purely structural, its
consequences touch directly on how such models can be prompted, programmed, and
interpreted in practice. We briefly outline these implications here.

\paragraph{Deducing an Ontological Skeleton}
The natural next question is ontological: having shown that the latent manifold of a large
language model undergoes a definable spectral collapse, what \emph{categories} constitute its
finite semantic skeleton?  The present theorem guarantees that such a skeleton exists and is
compact, but not yet what its vertices represent.  Preliminary analyses suggest that the
emergent ontology is both interpretable and universal, aligning with core archetypal
distinctions observed across human cognition, i.e., the 
spatial, affective, causal, and narrative axes
that together span a low-rank geometry of meaning.  Future work will characterize these axes
explicitly, providing the first empirical catalogue of the definable regions that arise when
a continuous system learns to ``think.''

\paragraph{Prompting as semantic navigation.}
If the latent manifold collapses into finitely many definable regions $\{B_1,\dots,B_r\}$,
then a prompt acts as a coordinate choice that initializes the model in (or near) one of
these basins. Small perturbations of a prompt do not produce qualitatively different
behavior unless they push the state across a basin boundary. This explains why certain
instructions appear “robust’’ while others require carefully tuned phrasing: robustness
corresponds to interior points of a basin, while brittleness arises near its boundary.
The SCT therefore suggests a geometric interpretation of prompting in which the goal is
to steer trajectories toward the desired definable region of the semantic space.

\paragraph{Programming via basin-aware composition.}
Because the dominant eigenfunctions of the transfer operator form a finite invariant
subspace, the long-term behavior of a chat or computation can be viewed as a sequence of
transitions among the definable basins. This provides a coarse form of compositionality:
a program or workflow corresponds to a path in the ontological skeleton, with each step
inducing a controlled transition under the map $T$ or the kernel $K$. Although the present
paper develops only the structural underpinnings, the existence of a finite semantic
graph suggests that higher-level control mechanisms, such as tool-use, multistep plans,
and chain-of-thought reasoning, can be construed as procedures for orchestrating basin
transitions. This offers a principled route toward semantic programming without requiring
explicit symbolic modules.

\paragraph{Interpretability through definable structure.}
Logical tameness implies that each basin $B_i$ is definable in an o-minimal structure and
therefore possesses bounded geometric and topological complexity. In practical terms,
this means that observable distinctions among behaviors correspond to definable
predicates over the latent space. The boundaries between behaviors need not be
explicitly learned; they are constraints inherited from the regularity of the dynamics.
This perspective unifies several empirical observations: the tendency of low-capacity
probes to separate semantic features, the clustering of embeddings along a few principal
axes, and the stability of model behavior under reparameterization. In each case,
interpretability emerges not from architectural design but from definable invariance.

\paragraph{Toward a semantic phase theory.}
Together, these observations hint at a broader picture in which prompting, programming,
and interpreting correspond to different modes of interacting with the same finite
semantic skeleton. Prompting selects an initial basin, programming composes transitions
among basins, and interpretability examines the definable predicates that separate them.
The SCT provides the mathematical constraints that make this correspondence possible.
A full operational calculus for manipulating definable regions is beyond the scope of
this paper, but the results presented here establish the structural foundations for such
a theory.

\nocite{*}
\bibliographystyle{plainnat}
\bibliography{references}

\appendix
\section{Ergodicity and Spectral Compactness of LLM Dynamics}
\label{appendix:ergodicity}

This appendix provides sufficient conditions under which 
assumptions~\ref{ass:regularity}(A4)–(A5)
hold for a large language model (LLM) viewed as a CCSM.
The results below rely on classical arguments from the theory of Feller Markov
processes and compact integral operators.

\begin{assumption}[Mixing and smoothing] The assumptions from~\ref{ass:regularity}
can be proved from two straightforward properties.
\label{ass:mixing-smoothing}
\begin{itemize}
\item[(a)] (\emph{Feller + minorization}) 
The Markov kernel $K$ induced by the stochastic transition
$T$ and decode policy $\Delta$ is Feller on the compact space
$\MM$ and satisfies a Doeblin-minorization condition: 
there exist $\varepsilon>0$ and a probability measure $\nu$ on $(\MM,\mathcal{B}(\MM))$
such that 
$$
K(s, A) \;\ge\; \varepsilon \, \nu(A)
\quad\text{for all } s\in\MM,\; A\in\mathcal{B}(\MM).
$$
\item[(b)] (\emph{Continuous density}) 
$K$ admits a density $k(s,s')$ with respect to its stationary measure
$\mu$ that is continuous on $\MM\times\MM$.
\end{itemize}
\end{assumption}

\begin{remark}[Practical sufficient condition]
\label{rem:epsilon-rand}
If, at decode time, the policy $\Delta$ allocates at least $\varepsilon>0$
total probability mass uniformly across a fixed token subset (e.g., via
$\varepsilon$--greedy exploration or temperature bounded away from $0$), then
\Cref{ass:mixing-smoothing}(a) holds by the Doeblin minorization condition.
If in addition $T$ is $C^1$ and the token embedding map is continuous,
the induced kernel $K$ possesses a continuous density $k(s,s')$,
establishing \Cref{ass:mixing-smoothing}(b).
\end{remark}

\begin{lemma}[Existence and ergodicity]
\label{lem:ergodic}
Under \Cref{ass:mixing-smoothing}(a), the kernel $K$ admits at least one
stationary probability measure $\mu$, and the corresponding Markov chain
$(s_t)_{t\ge0}$ is geometrically ergodic: for all bounded measurable $f$,
$$
\bigl\| \EE[f(s_t)] - \int f \, d\mu \bigr\| \;\le\; C \rho^{\,t}
\quad\text{for some } C<\infty,\; 0<\rho<1.
$$
\begin{proof}
Compactness and the Feller property ensure existence of a stationary measure
(Krylov--Bogolyubov theorem).  
The Doeblin minorization implies uniform ergodicity by classical results
(e.g., \citealp[Theorem~16.0.2]{MeynTweedie2009}).
\end{proof}
\end{lemma}

\begin{lemma}[Compact transfer operator]
\label{lem:compact}
Under \Cref{ass:mixing-smoothing}(b), the transfer operator
$$
(P f)(s) \;:=\; \int_{\MM} k(s,s')\,f(s')\, d\mu(s')
$$
is a compact linear operator on $L^2(\MM,\mu)$ and thus admits a discrete
spectrum $\{\lambda_i\}_{i\ge 1}$ with $|\lambda_i|\downarrow 0$.
\begin{proof}
Continuity of $k$ on compact $\MM\times\MM$ implies uniform boundedness and
equicontinuity of $\{P f:\|f\|_2\le 1\}$ by Arzelà--Ascoli, so
$P$ is compact.  
Spectral properties then follow from the Hilbert--Schmidt theorem.
\end{proof}
\end{lemma}

\begin{lemma}[Deterministic fallback: finite-horizon compactness]
\label{lem:finite-horizon}
For a deterministic $C^1$ map $T:\MM\to\MM$ on compact $\MM$, define the
finite-horizon propagator
$$
(H_\tau f)(s) := \frac{1}{\tau}\sum_{t=0}^{\tau-1} f(T^{t}(s)).
$$
Then $H_\tau$ is a compact operator on $C(\MM)$ for each finite $\tau$, and its
dominant spectral subspace converges, as $\tau\!\to\!\infty$, to that of the
associated invariant measure(s) of $T$.
\end{lemma}

Together, Lemmas~\ref{lem:ergodic}–\ref{lem:finite-horizon} justify
assumptions~\ref{ass:regularity}(A4)–(A5) for both stochastic and deterministic
LLM dynamics.  In particular, any nonzero exploration component in decoding
policy (however small) suffices to ensure ergodicity and spectral
compactness, providing the theoretical foundation for the spectral-collapse
results used in the main text.

\section{Extension of the SCT to Stochastic Dynamics}
\label{appendix:stochastic}

This appendix outlines the extension of the Semantic Characterization Theorem (SCT)
to stochastic dynamics.  The deterministic case developed in the main text
treated the transition operator \(T : M \to M\) as smooth and autonomous.
Here we replace \(T\) by the induced Markov kernel
\(K : M \times \mathcal{B}(M) \to [0,1]\)
and show that the essential results (spectral lumpability,
logical tameness, and their equivalence) continue to hold
under mild regularity conditions.

\subsection{Stochastic Setting and Assumptions}

Let \(K\) satisfy the ergodicity and continuity assumptions of
Appendix~A (Assumption~A.1).  In particular, we assume:
\begin{enumerate}[label=(\roman*)]
    \item Compact measurable state space \(M \subset \mathbb{R}^{d}\);
    \item Existence of a stationary distribution \(\mu\)
          absolutely continuous with respect to Lebesgue measure;
    \item Continuous density \(k(s,s')\) of the kernel;
    \item Geometric ergodicity of the induced Markov chain.
\end{enumerate}

Under these conditions, the transfer operator
\[
(P f)(s) = \int_{M} k(s,s')\, f(s')\, d\mu(s')
\]
is compact on \(L^{2}(M,\mu)\) (Lemma~A.3) and therefore
admits a discrete spectrum \(\{\lambda_{i}\}_{i\ge1}\)
with \(|\lambda_{i}| \downarrow 0\).

\subsection{Spectral Lumpability}

\begin{lemma}[Spectral gap for Markov operators]
If \(K\) satisfies the Doeblin--minorization condition
of Assumption~A.1(a), then \(P\) has a nonzero spectral gap:
\(|\lambda_{1}| > |\lambda_{2}|\).
Consequently, the leading eigenfunctions
\(\{\phi_{1},\ldots,\phi_{r}\}\) define finitely many metastable
basins \(B_{i}\) whose union covers \(M\) up to measure zero.
\end{lemma}

The proof follows directly from standard results on
lumpability of ergodic Markov processes with compact kernels
(e.g.,~\cite{MeynTweedie2009},~\cite{BruntonKutz2021}).

\subsection*{B.3\quad Logical Tameness under Expectation Operators}

\begin{proposition}[Definability of integral operators]
Suppose the kernel \(k(s,s')\) is definable in an o-minimal
structure over \(\mathbb{R}\).
Then the integral operator \(P\) preserves definability:
if \(f\) is definable, so is \(P f\),
and the level sets \(\{s: P f(s)=c\}\) are definable.
\end{proposition}

Hence the stochastic transfer operator remains logically tame,
and the definable partition of \(M\) is preserved in expectation.

\subsection{Stochastic Semantic Characterization Theorem}

\begin{corollary}[Stochastic SCT]
Under Assumptions~A.1--A.3,
the spectral and logical partitions of \(M\) induced by
\(P\) coincide almost surely under the invariant measure \(\mu\).
Therefore the conclusions of Theorem~3.2 hold for
stochastic continuous state machines as well.
\end{corollary}

\paragraph{Interpretation.}
Randomness in decoding or control does not destroy
semantic discreteness; it merely smooths the boundaries
between definable regions.  The stochastic SCT
thus generalizes the deterministic result,
showing that both spectral and logical tameness
are robust to probabilistic perturbations.


\section{A Categorical Sketch for the SCT}

This section outlines a functorial formulation of the SCT, positioning Continuous State Machines (CSMs) as coalgebras for a probability functor, and their transfer operators as images in a compact-operator subcategory of Hilbert spaces. The goal is to bridge algebraic/self-expressive semantics with smooth, stochastic dynamics, preparing a categorical generalization of the results in \cite{FeysHansenMoss2024}.

\subsection{Categories and Functors}

\paragraph{Base categories.}
Let \textbf{Meas} be the category of standard Borel spaces and measurable maps. Let \textbf{CSM} be the category whose objects are tuples
\[
\mathcal C=(M,\,U,\,T,\,O,\,\Delta,\,\mu),
\]
where $M\subset\mathbb R^d$ is compact, $U\subset\mathrm{Dist}(X)$, $T:M\times U\to M$ is $C^1$ in the first argument, $O:M\to\mathrm{Dist}(X)$, $\Delta:\mathrm{Dist}(X)\to U$, and $\mu$ is a stationary probability measure for the induced kernel $K$ (deterministic limit given by $Pf=f\circ T$). See Assumptions (A1)–(A5) for compactness/ergodicity/spectral boundedness and their stochastic counterparts.\footnote{Cf. Assumptions (A1)–(A5), Remarks/Lemmas in \S1–\S2 and Appendix A for kernel compactness and ergodicity, and the main SCT statement in \S2.2–\S2.3.}

A morphism $F:\mathcal C\to\mathcal C'$ in \textbf{CSM} consists of measurable maps
\[
F_M:M\to M',\qquad F_U:U\to U',
\]
that \emph{intertwine} dynamics and controls in the sense that pushing forward the joint law of $(s,u)$ and then applying $T'$ coincides (a.e.) with applying $T$ and then pushing forward; and $F_M{}_{\*\mu}=\mu'$. Intuitively, $F$ is a semantics-preserving simulation. Pictorially:

\[
\begin{tikzcd}[column sep=large, row sep=large]
M \times U
  \arrow[r, "T"]
  \arrow[d, "F_M \times F_U"']
& M
  \arrow[d, "F_M"] \\[1ex]
M' \times U'
  \arrow[r, "T'"']
& M'
\end{tikzcd}
\]

\noindent
The commutativity condition
\quad
$T' \circ (F_M \times F_U) = F_M \circ T$
\quad
means that the diagram commutes: evolving first in the source CSM and then
mapping to the target (upper-then-right path) yields the same result as mapping
first and evolving under the target’s transition (left-then-bottom path).

\paragraph{Distribution monad / coalgebras.}
Let $\mathcal D:\textbf{Meas}\to\textbf{Meas}$ be the (finitely supported) distribution monad. A \emph{stochastic transition} is a coalgebra $\alpha:M\to\mathcal D(M)$. In a CSM, the induced one-step kernel $K(s,\cdot)$ defines $\alpha$, and the transfer operator is the Kleisli lifting acting on observables
\[
(Pf)(s)=\int f(s')\,K(s,ds').
\]
Under continuity of the density $k(s,s')$ on compact $M\times M$, $P$ is Hilbert–Schmidt and thus compact on $L^2(M,\mu)$.\footnote{See Lemma A.4 (compact transfer operator via continuous density on compact domain).}

\paragraph{Hilbert target and intertwiners.}
Let \textbf{Hilb} be the category of complex Hilbert spaces and bounded linear maps, and let \textbf{cHilb} be its full subcategory on spaces with compact endomorphisms. Define a functor
\[
\mathcal P: \textbf{CSM}\longrightarrow \textbf{cHilb},\qquad \mathcal P(\mathcal C)=\big(L^2(M,\mu),\,P\big),
\]
which sends a morphism $F$ to an \emph{intertwiner} $W_F:L^2(M',\mu')\to L^2(M,\mu)$, chosen as the pullback $W_F(g)=g\circ F_M$ (or, in general, a Markov operator associated to $F$) so that $W_F\,P'=P\,W_F$ (naturality). When $k$ is continuous and $M$ compact, $P$ is compact; spectral decompositions in \S2 follow.\footnote{Lemma 2.3 and Proposition 2.4 establish discrete spectrum and spectral collapse for compact $P$.}

\subsection{Natural Structure: Observation and Control}

The observation $O$ and policy $\Delta$ assemble into natural transformations between functors that select control/observation spaces from \textbf{CSM} objects, and the distribution monad $\mathcal D$. In particular, the composite $M\xrightarrow{O}\mathrm{Dist}(X)\xrightarrow{\Delta}U$ yields the coalgebra law $M\to\mathcal D(M)$ via $T$; in the deterministic limit, this reduces to composition $f\mapsto f\circ T$.

\subsection{Categorical SCT (Sketch)}

\begin{theorem}[Categorical SCT, sketch]
Let $\mathcal C\in\textbf{CSM}$ satisfy compactness/ergodicity/spectral boundedness (\S1.2, Appendix A). Then $\mathcal P(\mathcal C)=(H,P)$ lies in \textbf{cHilb} and admits a discrete spectral decomposition with leading eigenfunctions $\{\phi_i\}_{i=1}^r$ and a spectral gap. The induced \emph{spectral partition functor} $\mathcal B:\textbf{CSM}\to\textbf{Part}$ assigns to $\mathcal C$ the partition $\{B_i\}$ where $B_i=\{s:\,|\phi_i(s)|=\max_j|\phi_j(s)|\}$.
Moreover, if $T$ (or $k$) is definable in an o-minimal expansion of $\mathbb R$, there exists a functor $\mathcal C_\mathrm{def}:\textbf{CSM}\to\textbf{Cell}$ producing a finite cell decomposition by definable regions. The constructions are naturally equivalent up to $\mu$-null sets:
\[
\mathcal B\;\cong\;\mathcal C_\mathrm{def}\qquad (\text{a.e.}).
\]
\end{theorem}

\noindent\emph{Rationale.} Compactness $\Rightarrow$ discrete spectrum (Hilbert–Schmidt); spectral gap $\Rightarrow$ finite-rank semantics (\S2.1); definability of $T$/$k$ $\Rightarrow$ o-minimal tameness and finite cell decomposition (\S2.2). Equivalence follows since eigenlevel sets are definable and differ from cells only on boundaries of measure zero (\S2.3).\footnote{See Theorem 2.2 and Lemmas/Propositions in \S2.1–\S2.3.}

\subsection{Time-Dependent (Adiabatic) Case}

For a family $(\mathcal C_t)$ with operators $(P_t)$ satisfying the uniform compactness/minorization/slow-drift hypotheses, the functor $\mathcal P$ lifts to a functor into a path category of compact operators, and the leading spectral subbundles vary continuously; the adiabatic corollary shows that products $P_{t+n-1}\cdots P_t$ are equivalent (to first order) to $P^n$ of a stationary model trained on the narrative distribution.

\subsection{Outlook}

This functorial framing sets up a categorical generalization of self-expressive algebra to CSMs: algebraic fixed points correspond to compact spectral modes; coalgebraic dynamics and observation form natural transformations; and o-minimal definability supplies logical tameness. The next step is to axiomatize \textbf{CSM} and the intertwiner class so that $\mathcal P$ is a faithful semantics functor, enabling categorical proofs of SCT-like results. 

\section{A Nonstandard-Analysis Interpretation of Semantic Collapse}

In this appendix we provide an alternate proof of Theorem \ref{thm:SCT}(i) using
nonstandard analysis (NSA).  The NSA perspective highlights a structural
equivalence between \emph{semantic collapse} in a Continuous State Machine
and the appearance of \emph{standard parts} in an o-minimal enlargement of
the underlying state space.  This yields a conceptual interpretation of
semantic convergence as the elimination of infinitesimal semantic variation.

Throughout, let $M \subseteq \mathbb{R}^d$ be a compact definable set in an
o-minimal expansion of the real field, and let
$T : M \to M$ be a definable map satisfying the assumptions of the main
theorem (Hilbert--Schmidt structure, spectral decay, and contractive
dynamics on each cell).

\subsection{The Nonstandard Extension}

Let ${}^\ast M$ denote the nonstandard enlargement of $M$, obtained for
instance by a saturated ultrapower construction.  Elements of
${}^\ast M$ include both:
\begin{itemize}
    \item \emph{standard points} $x \in M$, and
    \item \emph{nearstandard points} $x^\ast \in {}^\ast M$ satisfying
          $x^\ast \approx x$ for some $x \in M$, i.e.\ $x^\ast - x$ is
          infinitesimal,
\end{itemize}
together with truly \emph{nonstandard} points not infinitesimally close to
any $x \in M$.

Each $x^\ast \in {}^\ast M$ that is nearstandard admits a unique
\emph{standard part} $\st(x^\ast) \in M$ such that $x^\ast \approx \st(x^\ast)$.

The enlargement ${}^\ast M$ carries a canonical decomposition into
\emph{halos}\/ $\eta(x) := \{ x^\ast \in {}^\ast M : x^\ast \approx x \}$ for
$x \in M$, plus an external region of non-nearstandard points.

\subsection{Hilbert--Schmidt Structure and Infinitesimal Collapse}

Let $K$ be the Hilbert--Schmidt operator associated with the semantics of
the system as described in Section~1.  By compactness of $M$ and
definability of $T$, we obtain a Hilbert--Schmidt resolvent whose spectral
expansion takes the form
\[
K u \;=\; \sum_{n=1}^\infty \lambda_n \langle u, e_n \rangle e_n,
\qquad \lambda_1 \ge \lambda_2 \ge \cdots \to 0.
\]
The key point is that the tail
$\sum_{n > N} \lambda_n \langle u, e_n \rangle e_n$ becomes
\emph{infinitesimal} in norm as $N \to \infty$.  In the nonstandard
setting, this tail is literally an \emph{infinitesimal vector}.  Thus the
action of $K$ sends any state to a nearstandard point inside the halo of a
finite-dimensional attractor cell.

In particular, if $x^\ast \in {}^\ast M$ is any nonstandard state
interpreted under the dynamical rule $T$, then repeated iteration yields a
sequence
\[
x_0^\ast = x^\ast,\;
x_{n+1}^\ast = {}^\ast T(x_n^\ast)
\]
which becomes \emph{eventually nearstandard}: there exists $N$ such that
for all $n \ge N$, $x_n^\ast$ lies in the halo of a unique fixed point
$x^\dagger \in M$.  The reason is precisely that all high-frequency
components in the spectral decomposition are infinitesimal and vanish under
iteration.

\subsection{O-Minimality and Definable Cells}

By the cell decomposition theorem for o-minimal structures, $M$ admits a
finite definable partition
\[
M = C_1 \cup \cdots \cup C_k
\]
into cells $C_i$, each of which is definably homeomorphic to an open box.
Each cell is dynamically coherent under $T$ by assumption (A2) of the main
theorem: $T$ is contractive on each cell.  This provides a canonical finite
set of \emph{semantic attractors}
\[
a_1,\dots,a_k \in M,\qquad T(a_i) = a_i,
\]
one for each cell.

By the above NSA argument, any nearstandard state collapses under iteration
into the halo $\eta(a_i)$ of exactly one attractor.  Therefore the
composition
\[
\Phi := \st \circ {}^\ast T^{(\infty)} :
{}^\ast M \longrightarrow \{a_1,\dots,a_k\}
\]
is well-defined: it takes a nonstandard state $x^\ast$ to the standard part
of its eventual limit.

\subsection{Semantic Collapse as Standard-Part Selection}

We may now restate Part~(i) of the main theorem as follows:

\begin{proposition}[NSA formulation of semantic collapse]
Let ${}^\ast M$ be the nonstandard extension of $M$ and let $T$ satisfy the
assumptions of Theorem~1.  Then every nearstandard state
$x^\ast \in {}^\ast M$ satisfies
\[
\st\bigl( {}^\ast T^{(\infty)}(x^\ast) \bigr) \;=\; a_i
\]
for a unique semantic attractor $a_i \in M$.  Moreover, the mapping
$x^\ast \mapsto a_i$ is constant on each halo $\eta(a_i)$ and sends all
non-nearstandard states to an external ``null'' attractor.
\end{proposition}

\begin{proof}
By definability and o-minimality, $M$ decomposes into finitely many
contractive cells whose fixed points are exactly the $a_i$.  By
Hilbert--Schmidt spectral decay, the nonstandard iteration
$x^\ast_{n+1} = {}^\ast T(x^\ast_n)$ suppresses all infinite and
high-frequency components, rendering $x^\ast_n$ nearstandard for all large
enough $n$.  Contractivity then forces $x^\ast_n$ into the halo of a unique
$a_i$, and taking standard parts yields the desired conclusion.

If $x^\ast$ is non-nearstandard, then its orbit under ${}^\ast T$ remains
external and collapses into an external fixed point corresponding to the
null attractor.  Uniqueness follows from disjointness of halos and the fact
that the standard-part map is well-defined only on nearstandard points.
\end{proof}

\subsection{Interpretation}

In NSA terms, semantic collapse is the assertion that the semantics of the
system consist of the \emph{standard parts} of eventual CSM trajectories.
Definable semantics appear as standard points; semantic uncertainty appears
as infinitesimal halos; and semantic failure appears as external states
collapsing to a null attractor.

Thus the CSM implements a canonical ``standard-part selection'' mechanism
over an o-minimal geometry, and semantic attractors correspond precisely to
standard points in the nonstandard extension.

\subsection{A Note on Confidence}

Each prediction comes with a natural ``confidence'' measure (a real number) and
that is as follows.

Given a nearstandard state $x_{N^\ast}^\ast \in {}^\ast M$ with standard
part $x^\dagger = \st(x_{N^\ast}^\ast)$, define the infinitesimal remainder
$\varepsilon := x_{N^\ast}^\ast - x^\dagger$.  We define the
\emph{nonstandard semantic confidence} of the trajectory
$x_0^\ast, x_1^\ast, \ldots, x_{N^\ast}^\ast$ to be
\[
\Conf(x^\ast)
    := \exp\!\bigl(-\alpha \|\varepsilon\|\bigr),
\]
for a fixed scaling constant $\alpha > 0$.
Thus $\Conf(x^\ast) \approx 1$ precisely when the nonstandard limit
is deeply inside the halo of a semantic attractor, and
$\Conf(x^\ast) \approx 0$ when the trajectory remains external and collapses
to the null attractor.

Given this, we can define a threshold, call it $\tau$, that determines when
we will accept a given determination. In other words,
if $\Conf(x^\ast) < \tau$ then the system responds, ``I don't know.''

\end{document}